\documentclass[nohyperref]{article}

\usepackage{microtype}
\usepackage{graphicx}
\usepackage{subfigure}
\usepackage{booktabs} 
\usepackage{amsmath,amsthm,amsfonts,thmtools}
\usepackage{amssymb}
\usepackage{multirow}
\usepackage{xcolor}

\usepackage{mathrsfs} 

\usepackage{hyperref}

\newcommand{\bX}{{\bf X}}

\newcommand{\bw}{{\bf w}}
\newcommand{\bW}{{\bf W}}

\newcommand{\bb}{{\bf b}}

\newcommand{\bA}{{\bf A}}
\newcommand{\ba}{{\bf a}}
\newcommand{\bom}{{\bf w}}
\newcommand{\bY}{{\bf Y}}
\newcommand{\bF}{{\bf F}}

\newcommand{\bC}{{\bf C}}
\newcommand{\bc}{{\bf c}}

\newcommand{\cG}{\mathcal{G}}
\newcommand{\cN}{\mathcal{N}}
\newcommand{\cV}{\mathcal{V}}
\newcommand{\cE}{\mathbf{E}}

\newcommand{\ep}{\epsilon}
\newcommand{\ord}{{\mathcal O}}
\newcommand{\R}{{\mathbb R}}
\newcommand{\Dt}{{\Delta t}}

\newcommand{\fref}[1] {Fig.~\ref{#1}}
\newcommand{\Tref}[1]{Table~\ref{#1}}

\newcommand{\cJ}{{\bf J}}
\newcommand{\bZ}{{\bf Z}}
\newcommand{\ident}{\mathrm{I}}

\newcommand{\bD}{{\bf D}}
\newcommand{\bE}{{\bf E}}

\newtheorem{theorem}{Theorem}[section]

\newtheorem{proposition}[theorem]{Proposition}
\newtheorem{definition}[theorem]{Definition}

\newtheorem{remark}[theorem]{Remark}

\usepackage[accepted]{icml2022}

\icmltitlerunning{Graph-Coupled Oscillator Networks}

\begin{document}

\twocolumn[
\icmltitle{Graph-Coupled Oscillator Networks}



\icmlsetsymbol{equal}{*}

\begin{icmlauthorlist}
\icmlauthor{T. Konstantin Rusch}{eth,eth_ai}
\icmlauthor{Benjamin P. Chamberlain}{twi}
\icmlauthor{James Rowbottom}{twi}
\icmlauthor{Siddhartha Mishra}{eth,eth_ai}
\icmlauthor{Michael M. Bronstein}{ox,twi}
\end{icmlauthorlist}

\icmlaffiliation{eth}{Seminar for Applied Mathematics (SAM), D-MATH, ETH Z\"urich, Switzerland}
\icmlaffiliation{twi}{Twitter Inc., London, UK}
\icmlaffiliation{eth_ai}{ETH AI Center, ETH Z\"urich}
\icmlaffiliation{ox}{Department of Computer Science, University of Oxford, UK}

\icmlcorrespondingauthor{T. Konstantin Rusch}{konstantin.rusch@sam.math.ethz.ch}

\icmlkeywords{Machine Learning, ICML}

\vskip 0.3in
]



\printAffiliationsAndNotice{} 

\begin{abstract}
We propose Graph-Coupled Oscillator Networks (GraphCON), a novel framework for deep learning on graphs. It is based on discretizations of a second-order system of ordinary differential equations (ODEs), which model a network of nonlinear controlled and damped oscillators, coupled via the adjacency structure of the underlying graph. The flexibility of our framework permits any basic GNN layer (e.g. convolutional or attentional) as the coupling function, from which a multi-layer deep neural network is built up via the dynamics of the proposed ODEs. We relate the oversmoothing problem, commonly encountered in GNNs, to the stability of steady states of the underlying ODE and show that zero-Dirichlet energy steady states are not stable for our proposed ODEs. This demonstrates that the proposed framework mitigates the oversmoothing problem. Moreover, we prove that GraphCON mitigates the exploding and vanishing gradients problem to facilitate training of deep multi-layer GNNs. Finally, we show that our approach offers competitive performance with respect to the state-of-the-art on a variety of graph-based learning tasks.  
\end{abstract}

\section{Introduction}
Graph Neural Networks (GNNs) \cite{sperduti1994encoding,goller1996learning,sperduti1997supervised,frasconi1998general,gori2005new,scarselli2008graph,bruna2013spectral,chebnet,gcn,MoNet,mpnn} are a widely-used class of models for learning on relations and interaction data. These models have recently been successfully applied in a variety of tasks such as computer vision and graphics \cite{MoNet}, recommender systems \cite{ying2018graph}, transportation \cite{derrowpinion2021traffic}, computational chemistry \citep{mpnn},  drug discovery \cite{gaudelet2021utilizing}, physics \citep{shlomi2020graph}, and analysis of social networks (see \citet{zhou,gdlbook} for additional applications).

Several recent works proposed Graph ML models based on differential equations coming from physics \cite{Ave1,Poli2019,Zhu1,Xhonneux2020}, including diffusion \cite{grand} and wave \cite{pde-gcn} equations and  geometric equations such as Beltrami  \cite{blend} and Ricci \cite{topping2021understanding} flows. 
Such approaches allow not only to recover popular GNN models as discretization schemes for the underling differential equations, but also, in some cases, can address problems encountered in traditional GNNs such as oversmoothing \cite{os1,os2} 
and bottlenecks \cite{alon2020bottleneck}.

In this paper, we propose a novel 
physically-inspired approach to learning on graphs. 
Our framework, termed \textbf{GraphCON} (Graph-Coupled Oscillator Network) builds upon suitable time-discretizations of a specific class of ordinary differential equations (ODEs) that model the dynamics of a network of non-linear controlled and damped oscillators, which are coupled via the adjacency structure of the underlying graph.
Graph-coupled oscillators are often encountered in mechanical, electronic, and biological systems, and have been studied extensively \cite{stgz1}, with a prominent example being functional circuits in the brain such as cortical columns \cite{SE1}. In these circuits, each neuron oscillates with periodic firing and spiking of the action potential. The network of neurons is coupled in the form of a graph, with neurons representing nodes and edges corresponding to synapses linking neurons. 
\paragraph{Main Contributions.} In the subsequent sections, we will demonstrate the following features of GraphCON:
\begin{itemize}
    \item GraphCON is flexible enough to accommodate any standard GNN layer (such as GAT or GCN) as its coupling function. As timesteps of our discretized ODE can be interpreted as layers of a deep neural network \cite{neural,HR1,grand}, one can view GraphCON as a wrapper around any underlying basic GNN layer allowing to build deep GNNs. Moreover, we will show that standard GNNs can be recovered as steady states of the underlying class of ODEs, whereas GraphCON utilizes their dynamic behavior to sample a richer set of states, which leads to better expressive power.
    \item We mathematically formulate the frequently encountered oversmoothing problem for GNNs \cite{os1,os2} in terms of the stability of zero-Dirichlet energy steady states of the underlying equations. By a careful analysis of the dynamics of the proposed ODEs, we demonstrate that any zero-Dirichlet energy steady states are not (exponentially) stable. Consequently, we show that the oversmoothing problem for GraphCON is mitigated by construction. 
    
\item We rigorously prove that GraphCON mitigates the so-called exploding and vanishing gradients problem for the resulting GNN. Hence, GraphCON can greatly improve the trainability of deep multi-layer GNNs. 
    \item We provide an extensive empirical evaluation of GraphCON on a wide variety of graph learning tasks such as transductive and inductive node classification and graph regression and classification, demonstrating that GraphCON achieves competitive performance. 
\end{itemize}

\section{GraphCON}
Let $\mathcal{G}=(\mathcal{V},\mathcal{E}\subseteq \mathcal{V}\times \mathcal{V})$ be an undirected graph with $|\mathcal{V}|=v$ nodes and $|\mathcal{E}|=e$ edges consisting of unordered pairs of nodes $\{ i,j \}$ and denoted $i \sim j$. We will label nodes by the index $i \in \cV = \{1,2,\ldots, v\}$. For any $i \in \cV$, we denote its \emph{$1$-neighborhood} as 
$
\cN_i =
\{j \in \cV : i\sim j \}%
$. 
Furthermore, let $\bX \in \mathbb{R}^{v\times {m}}$ be given by $\bX = \{\bX_i\}$ for $i \in \cV$, denoting the ${m}$-dimensional feature vector at each node $i$. 

Central to our framework is a 
\emph{graph dynamical system} represented by the following nonlinear system of ODEs:\vspace{-2mm}
\begin{align}
\label{eq:cont_graphCON}
    \bX^{\prime\prime} = \sigma(\bF_\theta(\bX,t)) - \gamma\bX - \alpha \bX^\prime.\vspace{-2mm}
\end{align}
Here, $\bX(t)$ denotes the time-dependent ${v\times {m}}$-matrix of node features,  $\sigma$ is the activation function, $\bF_\theta$ is a general learnable (possibly time-dependent) 1-neighborhood coupling function of the form\vspace{-2mm} 
\begin{equation}
\label{eq:cpl}
\left(\bF_{\theta}(\bX,t)\right)_i = \bF_{\theta}\left(\bX_i(t),\bX_j(t),t\right) \quad \forall i\sim j, 
\vspace{-2mm}
\end{equation}
parametrized with a set of learnable parameters $\theta$.

By introducing the auxiliary \emph{velocity} variable $\bY(t) = \bX^\prime(t) \in \mathbb{R}^{v\times {m}}$, we can rewrite
the second-order ODEs \eqref{eq:cont_graphCON} as a first-order system:
\begin{equation}
\begin{aligned}
\label{eq:cont_graphCON_first_order}
    \bY^\prime &= \sigma(\bF_\theta(\bX,t)) - \gamma\bX - \alpha \bY, \\
    \bX^\prime &= \bY.
\end{aligned}\vspace{-2mm}
\end{equation}

The key idea of our framework is, given the input node features $\bX(0)$ as an initial condition, to use the solution $\bX(T)$ at some time $T$ as the output (more generally, one can also apply (linear) transformations (embeddings) to $\bX(0)$ and $\bX(T)$). As will be shown in the following section, the space of solutions of our system is a rich class of functions that can solve many learning tasks on a graph. 

The 
system \eqref{eq:cont_graphCON_first_order} must be solved by an iterative numerical solver using a suitable time-discretization. 
It is highly desirable for a time-discretization to preserve the structure of the underlying ODEs \eqref{eq:cont_graphCON_first_order} \cite{HW1}. 
In this paper, we use the following IMEX (implicit-explicit) time-stepping scheme, which extends the 
{\em symplectic Euler method} \cite{HW1} to systems with an additional damping term, \vspace{-1mm}
\begin{equation}
\begin{aligned}
\label{eq:disc_graphCON}
    \bY^n &= \bY^{n-1} + \Dt [\sigma(\bF_\theta(\bX^{n-1},t^{n-1})) \\&- \gamma\bX^{n-1} - \alpha \bY^{n-1}], \\
    \bX^n &= \bX^{n-1} + \Dt\bY^n,
\end{aligned}
\end{equation}
for $n=1,\dots,N$, where $\Dt>0$ is a fixed time-step and $\bY^n,\bX^n$ denote the hidden node features at time $t^n=n\Dt$. 
The iterative scheme~\eqref{eq:disc_graphCON} can be interpreted as an $N$-layer graph neural network (with potential additional linear input and readout layers, omitted here for simplicity), which we refer to as 
{\bf GraphCON} (see section \ref{sec:3} for the motivation of this nomenclature). 
The coupling function $\bF_\theta$ plays the role of a message passing  mechanism (\citet{mpnn}, also referred to, in various contexts, as `diffusion' or `neighborhood aggregation') in traditional GNNs. 

\paragraph{Choice of the coupling function $\bF_\theta$.} 
Our framework allows for any learnable 1-neighborhood coupling to be used as $\bF_\theta$, including instances of message passing mechanisms commonly used in the Graph ML literature such as GraphSAGE \citep{graphsage}, Graph Attention \cite{gat}, Graph Convolution \cite{chebnet,gcn}, SplineCNN \citep{splineCNN}, or MoNet  \citep{MoNet}). 
In this paper, we focus on two particularly popular choices: 
    
{\em Attentional message passing} of \citet{gat}: \vspace{-1mm}
    \begin{align*}
        \bF_\theta(\bX^n,t^n) = \bA^n(\bX^{n})\bX^{n}\bW^n,\vspace{-1mm}
    \end{align*}
    with learnable weight matrices $\bW^n \in \mathbb{R}^{{m}\times {m}}$ and attention matrices $\bA^n\in \mathbb{R}^{n\times n}$ following the adjacency structure of the graph $\mathcal{G}$, i.e., $(\bA^n(\bX^{n}))_{ij} = 0$ if $j \notin \mathcal{N}_i$ and  
    \begin{align*}
    &(\bA^n(\bX^{n}))_{ij} = \\&\frac{\exp({\rm LeakyReLU}(\ba^\top[\bW^n\bX^n_i||\bW^n\bX^n_j]))}{\sum\limits_{k \in \mathcal{N}_i 
    }\exp({\rm LeakyReLU}(\ba^\top[\bW^n\bX^n_i||\bW^n\bX^n_k]))},
\end{align*}
otherwise (here $\bX^n_i$ denotes the $i$-th row of $\bX^n$ and $\ba\in\mathbb{R}^{2{m}}$). 
%
    We refer to \eqref{eq:disc_graphCON} based on this attentional 1-neighborhood coupling as \textbf{GraphCON-GAT}.
    
    {\em Graph convolution} operator of \citet{gcn}:
    \begin{align}
    \label{eq:gcn}
        \bF_\theta(\bX^n,t^n) = \hat{\bD}^{-\frac{1}{2}}\hat{\bA} \hat{\bD}^{-\frac{1}{2}}\bX^n\bW^n,
    \end{align}
    with $\hat{\bA}=\bA + \mathbf{I}$ denoting the adjacency matrix of $\mathcal{G}$ with inserted self-loops, diagonal degree matrix $\bD=\mathrm{diag}(\sum_{l=1}^n \hat{\bA}_{kl})$, and $\bW^n_i \in \mathbb{R}^{{m}\times {m}}$ being learnable weight matrices.
    We refer to \eqref{eq:disc_graphCON} based on this convolutional 1-neighborhood coupling as \textbf{GraphCON-GCN}.
    
\paragraph{Steady States of GraphCON and relation to GNNs.} It is straightforward to see that the steady states $\bX^{\ast},\bY^{\ast}$ of the GraphCON dynamical system  \eqref{eq:disc_graphCON} with an \emph{autonomous} coupling function $\bF_\theta = \bF_\theta(\bX)$ (as in GraphCON-GAT or GraphCON-GCN) are given by $\bY^{\ast} \equiv {\bf 0}$ and 
\begin{equation}
    \label{eq:gcss}
    \bX^{\ast} = \frac{\Dt}{\gamma}\sigma(\bF_\theta(\bX^{\ast})).
\end{equation}
Using a simple fixed point iteration to find the steady states \eqref{eq:gcss} yields a multi-layer GNN of the form;\vspace{-2mm}
\begin{equation}
    \label{eq:gcss1}
     \bX^{n} = \frac{\Dt}{\gamma}\sigma(\bF_\theta(\bX^{n-1})), \quad {\rm for}~n=1,2,\ldots, N.
\end{equation}
We observe that (up to a rescaling by the factor $\Delta t/\gamma$) equation \eqref{eq:gcss1} corresponds to the update formula for any standard $N$-layer message-passing GNN \cite{mpnn}, including such popular variants as GAT \cite{gat} or GCN \cite{gcn}. 

Thus, this interpretation of GraphCON \eqref{eq:disc_graphCON} clearly brings out its relationship with standard GNNs. Unlike in standard multi-layer GNNs of the generic form \eqref{eq:gcss1} that can be thought of as steady states of the underlying ODEs \eqref{eq:cont_graphCON_first_order}, GraphCON evolves the underlying node features {\em dynamically in time}. 
Interpreting the multiple GNN layers as iterations at times $t^n = n \Dt$ in \eqref{eq:disc_graphCON}, we observe that the node features in GraphCON follow the trajectories of the corresponding dynamical system and  can explore a richer sampling of the underlying latent feature space, leading to possibly greater expressive power than standard GNNs \eqref{eq:gcss1}, which might remain in the vicinity of steady states. 

Moreover, this interpretation also reveals that, in principle, any GNN of the form \eqref{eq:gcss1} can be used within the GraphCON framework, offering a very flexible and broad class of architectures. Hence, one can think of GraphCON as an {\em additional wrapper} on top of any basic GNN layer allowing for a principled and stable design of deep multi-layered GNNs. In the following Section~\ref{sec:3}, we show that such an approach has several key advantages over standard GNNs.

\section{Properties of GraphCON}
\label{sec:3}
To gain some insight into the functioning of GraphCON \eqref{eq:disc_graphCON}, we start by
setting the hyperparameter $\gamma = 1$ and assuming that the $1$-neighborhood coupling $\bF_{\theta}$ is given by either the GAT or GCN type coupling functions. In this case, the underlying ODEs \eqref{eq:cont_graphCON_first_order} takes the following node-wise form, 
\begin{equation}
\label{eq:ODEnw1}
\begin{aligned}
\bX_i^{\prime} &= \bY_i, \\
\bY_i^{\prime} &= \sigma \left(\sum\limits_{j \in \cN_i} \bA_{ij} \bX_j \right) - \bX_i - \alpha \bY_i, 
\end{aligned}
\end{equation}
for all nodes $i \in \cV$, with $\bA_{ij} = \bA\left(\bX_i(t),\bX_j(t)\right)\in \R$ stemming from the attention or convolution operators. Furthermore, the matrices are \emph{right stochastic} i.e., the entries satisfy,\vspace{-2mm}
\begin{equation}
\label{eq:aij}
\begin{aligned}
0 \leq \bA_{ij} &\leq 1, \quad \forall j \in \cN_i, \quad \forall i \in \cV, \\
\sum\limits_{j\in\cN_i} \bA_{ij}&= 1, \quad \forall i \in \cV.
\end{aligned}
\end{equation}

\paragraph{Uncoupled case.} 
The simplest case of \eqref{eq:ODEnw1}, corresponds to setting $\sigma \equiv 0$ and $\alpha = 0$. In this case, all nodes are \emph{uncoupled} from each other and the solutions of the resulting ODEs are of the form,
\begin{equation}
    \label{eq:osc1}
    \bX_i(t) = \bX_i(0) \cos(t) + \bY_i(0)\sin(t).
\end{equation}
Thus, the dynamics of the ODEs \eqref{eq:cont_graphCON_first_order} in this special case correspond to a \emph{system of uncoupled oscillators}, with each node oscillating at unit frequency. 

\paragraph{Coupled linear case.}
Next, we introduce coupling between the nodes that are adjacent on the underlying graph $\cG$ and assume identity activation function $\sigma(x) = x$. In this case, \eqref{eq:ODEnw1} is a \emph{coupled linear system} and an exact closed form solution, such as \eqref{eq:osc1} may not be possible. However, we can describe the dynamics of \eqref{eq:ODEnw1} in the form of the following proposition (proved in {\bf SM} \ref{app:pf1}),
\begin{proposition}
\label{prop:1}
Let the node features $\bX,\bY$ evolve according to the ODEs \eqref{eq:ODEnw1} with activation function $\sigma = \mathrm{id}$ and time-independent matrix $\bA$ (e.g. $\bA_{ij} = \bA(\bX_i(0),\bX_j(0))$ using the initial features). 
Further assume that $\bA$ is \emph{symmetric} 
and $\alpha =0$. 
Then 
\begin{equation}
    \label{eq:ebal}
    \begin{aligned}
   & \sum\limits_{i\in\cV} \|\bY_i(t)\|^2 + \sum\limits_{i\in\cV} \sum\limits_{j \in \cN_i} \bA_{ij} \|\bX_i(t)-\bX_j(t)\|^2 \\
   &= \sum\limits_{i\in\cV} \|\bY_i(0)\|^2 + \sum\limits_{i\in\cV} \sum\limits_{j \in \cN_i} \bA_{ij} \|\bX_i(0)-\bX_j(0)\|^2,
    \end{aligned}
\end{equation}
holds for all $t > 0$. 
\end{proposition}
Thus, in this case, we have shown that the dynamics of the underlying ODEs \eqref{eq:ODEnw1} preserves the \emph{energy},
\begin{equation}
    \label{eq:en}
\mathscr{E}(t):= \sum\limits_{i\in\cV} \|\bY_i(t)\|^2 + \sum\limits_{i\in\cV} \sum\limits_{j \in \cN_i} \bA_{ij} \|\bX_i(t)-\bX_j(t)\|^2,
\end{equation}
and the trajectories of \eqref{eq:ODEnw1} are constrained to lie on a manifold of the node feature space, defined by the level sets of the energy. In particular, energy \eqref{eq:en} is not produced or destroyed but simply redistributed among the nodes of the underlying graph $\cG$. Thus, the dynamics of \eqref{eq:cont_graphCON_first_order} in this setting amounts to the motion of a \emph{linear system of coupled oscillators}.

\paragraph{General nonlinear case.}
In the general case, we have (i) a nonlinear activation function $\sigma$; (ii) time-dependent non-linear coefficients $\bA_{ij} = \bA(\bX_i(t),\bX_j(t))$; and (iii) possible unsymmetrical entries $\bA_{ij} \neq \bA_{ji}$. All these factors destroy the energy conservation property \eqref{eq:ebal} and can possibly lead to unbounded growth of the energy. Hence, we need to add some damping to the system. To this end, the damping term in \eqref{eq:ODEnw1} is activated by setting $\alpha > 0$. 
Moreover, $\gamma \neq 1$ corresponds to controlling frequencies of the nodes. Thus, the overall dynamics of the underlying ODEs \eqref{eq:cont_graphCON_first_order} amounts to the motion of a nonlinear system of {\em coupled, controlled and damped oscillators} with the coupling structure being that of the underlying graph. This explains our choice of the name, \emph{Graph-Coupled Oscillatory Neural Network} or {`GraphCON'} for short.

We illustrate the dynamics of GraphCON in \fref{fig:graphCON_figure}, where the model is applied to the graph of a molecule from the ZINC database \cite{zinc}, with features $\bX$ denoting the position of the nodes and they are propagated in time through the action of GraphCON \eqref{eq:disc_graphCON}.  The oscillatory behavior of the node features, as well as their dependence on the adjacency structure of the underlying graph can be clearly observed in this figure.  

\begin{figure}[h]
\begin{center}
\centerline{\includegraphics[width=0.9\columnwidth]{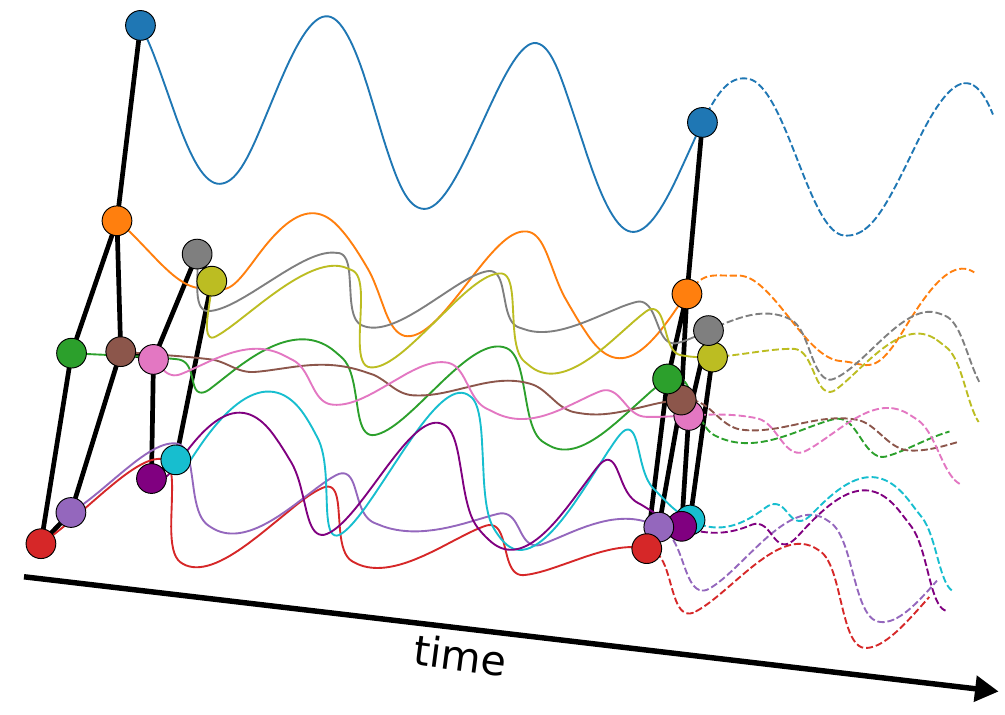}}\vspace{-2mm}
\caption{Illustration of GraphCON dynamics on a ZINC molecular graph. The initial positions of GraphCON ($\bX_0$ in \eqref{eq:disc_graphCON}) are represented by the 2-dimensional positions of the nodes, while the initial velocities ($\bY_0$ in \eqref{eq:disc_graphCON}) are set to the initial positions. The positions are propagated forward in time (`layers') using GraphCON-GCN with random weights. The molecular graph is plotted at initial time $t=0$ as well as at $t=20$.
}
\label{fig:graphCON_figure}\vspace{-2mm}
\end{center}
\vskip -0.2in
\end{figure}

\paragraph{Oversmoothing and GraphCON.}
One of the common plights of GNN models such as GAT \cite{gat}, GCN \cite{gcn} and their variants is  \emph{oversmoothing} \cite{os1,os2}, a 
phenomenon where all node features in a deep GNN converge to the same constant value as the number of hidden layers is increased. 
%
Consequently, one often must resort to shallow GNNs at the expense of expressive power  
\cite{os1,os2}. 
Many attempts have been made in recent years to mitigate the oversmoothing problem for GNNs, including regularization procedures such as DropEdge \cite{dropedge},  
using intermediate representations \cite{jknet}, or adding residual connections \cite{gcnii}. 

We will show that GraphCON allows to mitigate this problem by construction, and set off by formulating this problem in precise mathematical terms and to this end, we recall the \emph{Dirichlet energy}, defined on the node features $\bX$ of an undirected graph $\cG$ as,\vspace{-1mm}
\begin{equation}
\label{eq:graph_H1}
    \cE(\bX) = \frac{1}{v}\sum_{i \in \cV} \sum_{j \in \cN_i} \|\bX_i - \bX_j\|^2.\vspace{-1mm}
\end{equation}
Next, we define oversmoothing as follows:

\begin{definition}
\label{def:1}
Let $\bX^n$ denote the hidden features of the $n$th layer of an $N$-layer GNN, with $n=0,\dots,N$. 
We define {\em oversmoothing}  as the exponential convergence to zero of the layer-wise 
Dirichlet energy 
as a function of $n$, i.e.,\vspace{-1mm}
\begin{equation}
\label{eq:oversmoothing}
    \cE(\bX^n) \leq C_1e^{-C_2n},\vspace{-1mm}
\end{equation}
with some constants $C_1,C_2>0$.
\end{definition}
In other words, oversmoothing happens when the graph gradients vanish quickly (see for instance the illustration in \fref{fig:graph_h1_wavegnn_vs_gat}) in the number of hidden layers of the GNN. As a result, the feature vectors across all nodes rapidly (exponentially) converge to the same constant value. 
This behavior is commonly observed in GNNs and is identified as one of the reasons for the difficulty in designing deep GNNs.

GraphCON behaves rather differently and allows to 
mitigate the oversmoothing problem in the sense of definition \ref{def:1}. To see this, we focus on the underlying ODEs \eqref{eq:cont_graphCON_first_order}. It is trivial to extend the definition of oversmoothing from the discrete case to the continuous one by requiring that oversmoothing happens for the ODEs  \eqref{eq:cont_graphCON_first_order} if the 
Dirichlet energy behaves as, \vspace{-1mm}
\begin{equation}
    \label{eq:vgg1}
\cE(\bX(t)) \leq C_1 e^{-C_2t}, \quad \forall t > 0,\vspace{-1mm}
\end{equation}
for some $C_{1,2} > 0$. 

We have the following simple proposition (proved in {\bf SM} \ref{app:pf2}) that characterizes the oversmoothing problem for the underlying ODEs in the standard terminology of dynamical systems \cite{Wig1}, 
\begin{proposition}
\label{prop:2}
The oversmoothing problem occurs for the ODEs \eqref{eq:cont_graphCON_first_order} if and only if the hidden states $\left(\bX^\ast,\bY^{\ast}\right) = \left(\bc,{\bf 0}\right)$ are \emph{exponentially stable steady states (fixed points)} of the ODE \eqref{eq:cont_graphCON_first_order}, for some $\bc \in \R^{{m}}$ and ${\bf 0}$ being the ${m}$-dimensional vector with zeroes for all its entries. 
\end{proposition}
In other words, all the trajectories of the ODE \eqref{eq:cont_graphCON_first_order}, that start within the corresponding basin of attraction, have to converge exponentially fast in time  (satisfy \eqref{eq:vgg1}) to the corresponding steady state $\left(\bc, {\bf 0}\right)$ for the oversmoothing problem to occur for this system. Note that the basins of attraction will be different for different values of $\bc$.

Given this characterization, the key questions are a) whether $\left(\bc, {\bf 0}\right)$ are fixed points for the ODE \eqref{eq:cont_graphCON_first_order}, and b) whether these fixed points are exponentially stable. We answer these questions for the ODEs \eqref{eq:ODEnw1} in the following 

\begin{proposition}
\label{prop:3}
Assume that the activation function $\sigma$ in the ODEs \eqref{eq:ODEnw1} is ReLU. Then, for any $\bc \in \R^{{m}}$ such that each entry of the vector $\bc_\ell \geq 0$, for all $1 \leq \ell \leq m$, the hidden state $\left(\bc, {\bf 0}\right)$ is a steady state for the ODEs \eqref{eq:ODEnw1}. However under the additional assumption of $\alpha \geq \frac{1}{2}$, this fixed point is \emph{not exponentially stable}. 
\end{proposition}
The fact that $\left(\bc, {\bf 0}\right)$ is a steady state of \eqref{eq:ODEnw1}, for any \emph{positive} ${\bf c}$ is straightforward to see from the structure of \eqref{eq:ODEnw1} and the definition of the ReLU activation function. 
We can already observe from the energy identity \eqref{eq:ebal} for the simplified symmetric linear system that the energy \eqref{eq:en} for the small perturbations around the steady state $\left(\bc, {\bf 0}\right)$ is conserved in time. Hence, these small perturbations do not decay at all, let alone, exponentially fast in time. Thus, these steady states are not exponentially stable. 

An extension of this analysis to the nonlinear time-dependent, possibly non-symmetric system \eqref{eq:ODEnw1} is more subtle and the proof relies on the identity \eqref{eq:ebal1} (expressed in Proposition \ref{prop:4} in {\bf SM} \ref{app:pf3}) that describes how a suitably defined energy of the general system \eqref{eq:ODEnw1} evolves around small perturbations of the steady state $\left(\bc, {\bf 0}\right)$. A careful analysis of this identity reveals that these small perturbations can grow polynomially in time (at least for short time periods) and do not decay exponentially. Consequently, the fixed point $\left(\bc, {\bf 0}\right)$ is not stable. This shows that the oversmoothing problem, in the sense of definition \ref{def:1}, is mitigated for the ODEs \eqref{eq:cont_graphCON_first_order} and structure preserving time-discretizations of it such as \eqref{eq:disc_graphCON}, from which, in simple words it follows that { GraphCON} {\em mitigates oversmoothing by construction}. 

This analysis also illustrates the rich dynamics of \eqref{eq:cont_graphCON_first_order} as we show that even if the trajectories reach a steady state of the form $\left(\bc, {\bf 0}\right)$, very small perturbations will grow and the trajectory will veer away from this steady state, possibly towards other constant steady states which are also not stable. Thus, the trajectories can sample large parts of the latent space, contributing to the expressive power of the model. 

We remark here that the use of ReLU activation function in proposition \ref{prop:4} is purely for definiteness. Any other widely used activation function can be used in $\sigma$, with corresponding zero Dirichlet energy steady states being specified by the roots of the algebraic equation $\sigma(\bc) = \bc$ and an analogous result can be derived. For instance, the zero-Dirichlet energy steady state corresponding to the Tanh activation function is given by $\left({\bf 0},{\bf 0}\right)$.
\paragraph{On the exploding and vanishing gradients problem.} The mitigation of oversmoothing by GraphCON has a great bearing on increasing the expressivity of the resulting deep GNN. In addition, it turns out that using graph-coupled oscillators can also facilitate training of the underlying GNNs. To see this, we will consider a concrete example of the coupling function in \eqref{eq:disc_graphCON} to be GCN \eqref{eq:gcn}. Other coupling functions such as GAT can be considered analogously. For simplicity of exposition and without any loss of generality, we consider \emph{scalar node features} by setting $m=1$. We also set $\alpha,\gamma =1$. With these assumptions, a $N$-layer deep GraphCON-GCN reduces to the following explicit (node-wise) form, 
\begin{equation}
    \label{eq:gcon-gcn}
    \begin{aligned}
 \bY_i^n &= (1-\Dt) \bY_i^{n-1} + \Dt\sigma\left(\bC^{n-1}_i\right) - \Dt \bX^{n-1}_i, \\
 \bC^{n-1}_i &= \frac{\bom^n_i}{d_i} \bX^{n-1}_i + \sum\limits_{j \in \cN_i} \frac{\bom^n_j \bX^{n-1}_j}{\sqrt{d_id_j}}, \\
    \bX^n_i &= \bX^{n-1}_i + \Dt\bY^n_i, \quad \forall 1 \leq n \leq N, \quad \forall 1 \leq i \leq v.
\end{aligned}
\end{equation}
Here, $d_i = {\rm deg}(i)$, denoting the degree of a node $i \in \cV$ and $\bom^n\in \R^v$, denoting the learnable weight vector. 

Moreover, we are in a setting where the learning task is for the GNN to approximate the \emph{ground truth} vector $\overline{\bX} \in \R^v$. Consequently, we set up the following loss-function,

\begin{equation}
    \label{eq:lf}
    {\cJ}({\bw}) := \frac{1}{2v} \sum\limits_{i\in \cV} |{\bX}^N_i - \overline{\bX}_i|^2,
\end{equation}
with $\bw = [\bw^1,\bw^2,\cdots,\bw^N]$ denoting the concatenated learnable weights in \eqref{eq:gcon-gcn}. 
During training, one computes an approximate minimizer of the loss-function \eqref{eq:lf} with a (stochastic) gradient descent (SGD) procedure. At every step of gradient descent, we need to compute the gradient $\partial_\bw \cJ$. For definiteness, we fix node $k \in \cV$ and layer $1 \leq \ell \leq N$ and consider the learnable weight $\bw^\ell_k$. Thus, in a SGD step, one needs to compute gradient, $\frac{\partial \cJ}{\partial \bw^\ell_k}$. By chain rule, one readily proves the following identity (see for instance \cite{vanish_grad}),

\begin{equation}
    \label{eq:chain}
    \frac{\partial \cJ}{\partial \bw^\ell_k}
    = \frac{\partial \cJ}{\partial \bZ^N} \frac{\partial \bZ^N}{\partial \bZ^\ell} \frac{\partial \bZ^\ell}{\partial \bw^\ell_k}.
\end{equation} 
Here,
$$
Z^n = \left[\bX^n_1,\bY^n_1,\bX^n_2,\bY^n_2,\cdots,\bX^n_i,\bY^n_i,\cdots,\bX^n_v,\bY^n_v\right],
$$
is the concatenated node-feature vector at the layer $1 \leq n \leq N$. 

Furthermore, by using the product rule, we see that,
\begin{equation}
\label{eq:prod}
\frac{\partial {\bZ}^N}{\partial \bZ^\ell} = \prod\limits_{n=\ell+1}^N \frac{\partial {\bZ}^n}{\partial {\bZ}^{n-1}}.
\end{equation}
In other words, the gradient $\frac{\partial \cJ}{\partial \bw^\ell_k}$ measures the contribution made by the node $k$ in the $\ell$-th hidden layer to the learning process.

If we assume that the partial gradient behaves as $\frac{\partial {\bZ}^n}{\partial {\bZ}^{n-1}} \sim \lambda$, for all $n$, then, the long-product structure of \eqref{eq:prod} implies that $\frac{\partial {\bZ}^N}{\partial {\bZ}^\ell} \sim \lambda^{N-\ell}$. If on average, $\lambda > 1$, then we observe that the total gradient \eqref{eq:chain} can grow \emph{exponentially} in the number of layers, leading to the exploding gradients problem. Similarly, if on average, $\lambda < 1$, then the total gradient \eqref{eq:chain} can decay \emph{exponentially} in the number of layers, leading to the vanishing gradients problem. Either of these situations can lead to failure of training as the gradient step either blows up or does not change at all. Hence, for very deep GNN architectures, it is essential to investigate if the exploding and vanishing gradients problem can be mitigated. We start by showing the following upper bound (proved in {\bf SM} \ref{app:evgp}) on the gradients,
\begin{proposition}
\label{prop:gub}
Let $\bX^n,\bY^n$ be the node features, generated by Graphcon-GCN \eqref{eq:gcon-gcn}. We assume that $\Dt << 1$ is chosen to be sufficiently small. Then, the gradient of the loss function $\cJ$ \eqref{eq:lf} with respect to any learnable weight parameter $\bw^\ell_k$, for some $1 \leq k \leq v$ and $1 \leq \ell \leq N$ is bounded as 
\begin{equation}
\label{eq:gbd}
\begin{aligned}
  &\left| \frac{\partial \cJ}{\partial \bw^\ell_k} \right|   \leq \frac{\beta^{\prime}\hat{D}\Dt(1+\Gamma N \Dt)}{v} \left(\max\limits_{1 \leq i \leq v} (|\bX^0_i| + |\bY^0_i|)\right) \\
  &+ \frac{\beta^{\prime}\hat{D}\Dt(1+\Gamma N \Dt)}{v}\left(\max\limits_{1 \leq i \leq v} |\overline{\bX}_i| + \beta \sqrt{N\Dt}\right)^2.
  \end{aligned}
\end{equation}
Here, 
\begin{equation}
\label{eq:def1}
\begin{aligned}
\beta &= \max\limits_{x} |\sigma(x)|, \quad \beta^{\prime} = \max\limits_{x} |\sigma^{\prime}(x)|, \\
\hat{D} &= \max\limits_{i,j \in \cV} \frac{1}{\sqrt{d_id_j}}, \quad \Gamma:= 6 + 4\beta^{\prime}\hat{D}\max\limits_{1 \leq n \leq N} \|\bw^n\|_1.
\end{aligned}
\end{equation}
\end{proposition}
The upper bound \eqref{eq:gbd} clearly shows that the total gradient is globally bounded, independent of the number of layers $N$, if $\Dt \sim N^{-1}$, thus mitigating the exploding gradients problem. Even if the small parameter $\Dt$ is chosen independently of the number of layers $N$, the total gradient in \eqref{eq:gbd} only grows, at most quadratically in the number of layers, thus preventing exponential blowup of gradients and mitigating the exploding gradients problem. However, this upper bound \eqref{eq:gbd} does not necessarily rule out the vanishing gradients problem. To this end, we derive the following formula (in {\bf SM} \ref{app:evgp}) for the gradients, 
\begin{proposition}
\label{prop:glb}
For $1 \leq n \leq N$, let $\bX^n$ be the node features generated by GraphCON-GCN \eqref{eq:gcon-gcn}, Then for sufficiently small $\Dt << 1$, the gradient $\frac{\partial \cJ}{\partial \bw^\ell_k}$, for any $\ell,k$ satisfies the following expression,
\begin{equation}
\begin{aligned}
    \label{eq:gform}
    \frac{\partial \cJ}{\partial \bw^\ell_k} = &\frac{2\Dt^2}{v} \sum\limits_{j \in \cN_k} \frac{\sigma^{\prime}(\bC^{\ell-1}_j)\bX^{\ell-1}_j \left(\bX^N_j - \overline{\bX}_j\right)}{\sqrt{d_jd_k}} \\&+ \ord(\Dt^3),
\end{aligned}
\end{equation}
with the order notation being defined in {\bf SM} Eqn. \eqref{eq:ord}.
\end{proposition}
One readily observes from the formula \eqref{eq:gform}, that to leading order in the small parameter $\Dt$, the gradient $\frac{\partial \cJ}{\partial \bw^\ell_k}$ is \emph{independent} of the number of layers $N$ of the underlying GNN. Thus, although the gradient can be small (due to small $\Dt$), it will not vanish by increasing the number of layers, mitigating the vanishing gradient problem.

\section{Related Work} 
Differential equations have historically played a role in designing and interpreting various algorithms in machine learning, including non-linear dimensionality reduction methods \cite{belkin2003laplacian,coifman2006diffusion} and ranking \cite{page1999pagerank,chakrabarti2007dynamic} (all of which are related to closed-form solutions of diffusion PDEs). 
%
%
In the context of Deep Learning, differential equations have been used to derive various types of neural networks
including Neural ODEs and their
variants, 
that have been used to design and interpret residual \cite{neural} and convolutional \cite{HR1} neural networks. 
These approaches have recently gained traction in Graph ML, e.g. with ODE-based models for learning on graphs  \cite{Ave1,Poli2019,Zhu1,Xhonneux2020}.

\citet{grand} used parabolic diffusion-type PDEs to design GNNs 
using graph gradient and divergence operators as the spatial differential operator, a transformer type-attention as a learnable diffusivity function (`$1$-neighborhood coupling' in our terminology), and a variety of time stepping schemes to discretize the temporal dimension in this framework.  
\citet{blend} applied a non-euclidean diffusion equation (`Beltrami flow') to a joint positional-feature space, yielding a scheme with adaptive spatial derivatives (`graph rewiring'), and \citet{topping2021understanding} studied a discrete geometric PDE similar to Ricci flow to improve information propagation in GNNs. We can see the contrast between the diffusion-based methods of \citet{grand,blend} and GraphCON in the simple case of identity activation $\sigma(x) = x$.
Then, under the further assumption that the second-order time derivative $\bX^{\prime \prime}$ is removed from \eqref{eq:cont_graphCON} and $\alpha = \gamma = 1$, we recover the graph diffusion-PDEs of \cite{grand}. Hence, the presence of the temporal second-order derivative distinguishes this approach from diffusion-based PDEs. 

 \citet{pde-gcn} proposed a GNN framework arising from a mixture of parabolic (diffusion) and hyperbolic (wave) PDEs on graphs with convolutional coupling operators, which describe dissipative wave propagation. 
 We point out that a particular instance of their model (damped wave equation, also called as the {\em Telegrapher's equation}) can be obtained as a special case of our model  \eqref{eq:cont_graphCON} with 
 the identity activation function. 
 This is not surprising as the zero grid-size limit of oscillators on a regular grid yields a wave equation. However, given that we use a nonlinear activation function and the specific placement of the activation layer in \eqref{eq:cont_graphCON_first_order}, a local PDE interpretation of the general form of our underlying ODEs \eqref{eq:cont_graphCON} does not appear to be feasible. 

Finally, the explicit use of networks of coupled, controlled oscillators to design machine learning models was proposed in context of recurrent neural networks (RNNs) by \citet{cornn,unicornn}.

\section{Experimental results}
\label{sec:5}
We present a detailed experimental evaluation of the proposed framework on a variety of graph learning tasks. We test two settings of GraphCON: GraphCON-GCN (using graph convolution as the 1-neighborhood coupling in \eqref{eq:disc_graphCON}) and GraphCON-GAT (using the attentional coupling). Since in most experiments, these two configurations already outperform the state-of-the-art (SOTA), we only apply GraphCON with more involved coupling functions in a few particular tasks. All code to
reproduce our results can be found at \href{https://github.com/tk-rusch/GraphCON}{https://github.com/tk-rusch/GraphCON}.

\subsection{Evolution of Dirichlet Energy.}
We start by illustrating the dynamics of the Dirichlet energy \eqref{eq:graph_H1} of GraphCON for an undirected graph representing a 
2-dimensional $10\times10$ regular grid 
with 4-neighbor connectivity. 
The node features $\bX$ are randomly sampled from $\mathcal{U}([0,1])$ and then propagated through 100-layer GNNs (with random weights): GAT, GCN, and their GraphCON-stacked versions (GraphCON-GAT and GraphCON-GCN) for two different values of the damping parameter $\alpha=0,0.5$ in \eqref{eq:disc_graphCON} and with fixed $\gamma=1$. 
%
In 
\fref{fig:graph_h1_wavegnn_vs_gat}, we plot the (logarithm of)
Dirichlet energy 
of each layer's output with respect to (logarithm) of the layer number.  
It can clearly be seen that GAT and GCN suffer from the oversmoothing problem as the Dirichlet energy 
converges exponentially fast to zero, indicating that the node features become constant, while 
GraphCON is devoid of this behavior. This holds true even for non-zero value of the damping parameter $\alpha$, where the Dirichlet energy stabilizes after an initial decay.  
\begin{figure}[h]
\vskip 0.2in
\begin{center}
\centerline{\includegraphics[width=\columnwidth]{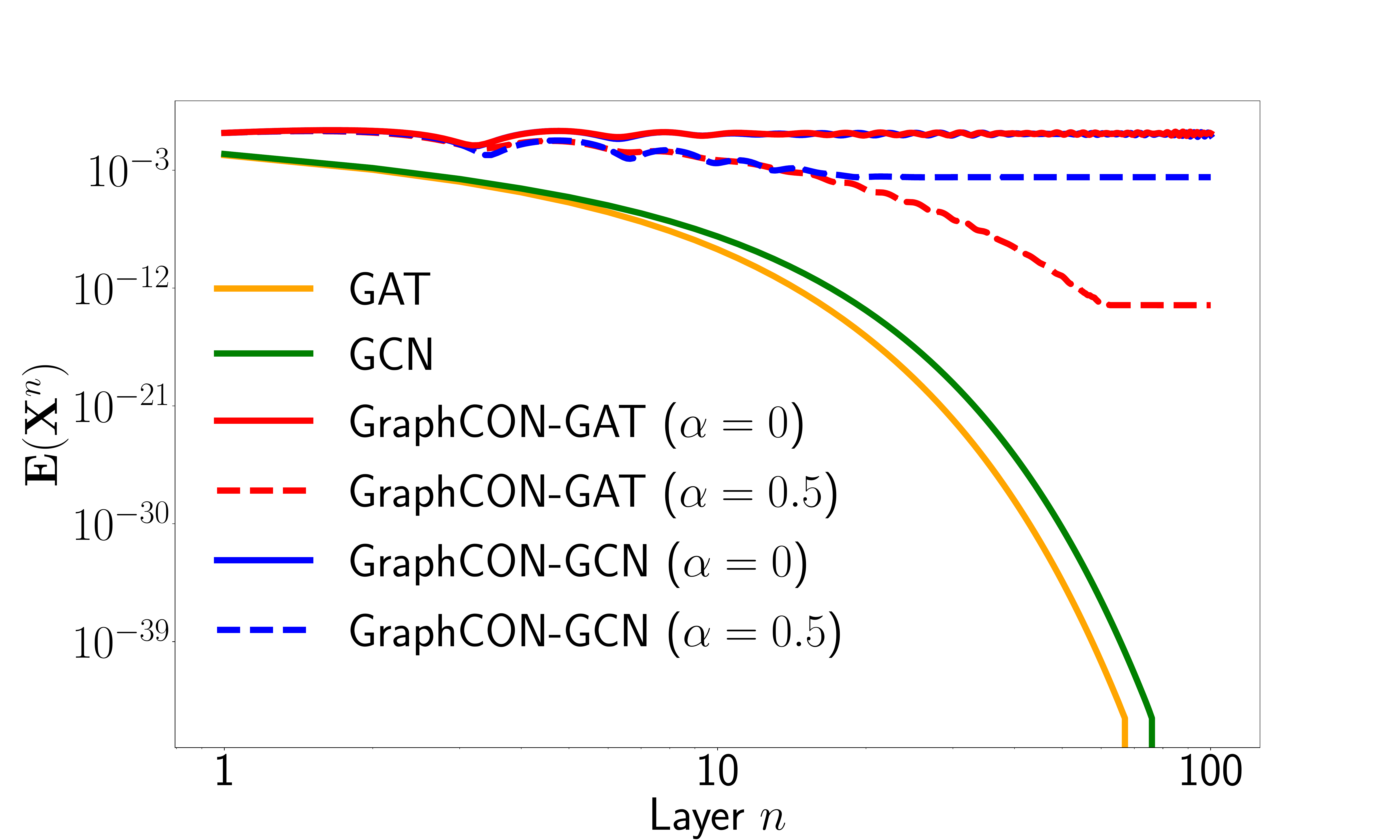}}
\caption{Dirichlet energy $\cE(\bX^n)$ of layer-wise node features $\bX^n$ propagated through a GAT and GCN as well as their GraphCON-stacked versions (GraphCon-GAT and GraphCON-GCN) for two different values of $\alpha=0,0.5$ in \eqref{eq:disc_graphCON} and fixed $\gamma=1$. 
}
\label{fig:graph_h1_wavegnn_vs_gat}
\end{center}
\vskip -0.2in
\end{figure}

\subsection{Transductive node classification}
We evaluate GraphCON on both homophilic and heterophilic datasets, where high homophily implies that the features in a node are similar to those of its neighbors. The homophily level reported in \Tref{tab:citation} and \Tref{tab:heterophil} is the measure proposed by \citet{pei2020geom}. 
\paragraph{Homophilic datasets.}
We consider three widely used node classification tasks, based on the citation networks Cora \citep{cora}, Citeseer \citep{citeseer} and Pubmed \citep{pubmed}. 
We follow the evaluation protocols and training, validation, and test splits of  \citet{shchur,grand}, using only on the largest connected component in each network. 

\Tref{tab:citation} compares GraphCON with standard GNN baselines: GCN \citep{gcn}, GAT \citep{gat}, MoNet \citep{MoNet}, GraphSAGE (GS) \citep{graphsage}, CGNN \citep{CGNN}, GDE \citep{GDE}, and GRAND \cite{grand}. We observe that GraphCON-GCN and GraphCON-GAT outperform pure GCN and GAT consistently. 
We also provide results for GraphCON based on the propagation layer used in GRAND i.e., transformer \cite{vaswani_attention} based graph attention, referred to as GraphCON-Tran, which also outperforms the basic underlying model. Overall, GraphCON models show the best performance on all these datasets. 
\begin{table}[h]
\caption{Transductive node classification test accuracy (MAP in \%) on homophilic datasets. Mean and standard deviation are obtained using 20 random initializations on 5 random splits each. The three best performing methods are highlighted in {\bf \textcolor{red}{red}} (First), {\bf \textcolor{blue}{blue}} (Second), and {\bf \textcolor{violet}{violet}} (Third).
}
\label{tab:citation}
\vskip 0.15in
\begin{center}
\begin{small}
\resizebox{\columnwidth}{!}{
\begin{tabular}{lccc}
\toprule
 &  \textbf{Cora} & \textbf{Citeseer} & \textbf{Pubmed} \\
{\em Homophily level} & $\bf 0.81$ & $\bf 0.74$ & $\bf 0.80$ \\
\midrule
GAT-ppr & $81.6 \pm 0.3$ & $68.5 \pm 0.2$ & $76.7 \pm 0.3$\\
MoNet & $81.3 \pm 1.3$ & $71.2 \pm 2.0$ & $78.6 \pm 2.3$\\
GraphSage-mean & $79.2 \pm 7.7$ & $71.6 \pm 1.9$ & $77.4 \pm 2.2$\\
GraphSage-maxpool & $76.6 \pm 1.9$ & $67.5 \pm 2.3$ & $76.1 \pm 2.3$\\
CGNN & $81.4 \pm 1.6$ & $66.9 \pm 1.8$ & $66.6 \pm 4.4$\\
GDE & $78.7 \pm 2.2$ & $71.8 \pm 1.1$ & $73.9 \pm 3.7$\\
\midrule[0.005em]
GCN & $81.5 \pm 1.3$ & $71.9 \pm 1.9$ & $77.8 \pm 2.9$\\
{\bf GraphCON}-GCN & $81.9 \pm 1.7$ & $72.9 \pm 2.1$ & $\bf \color{violet} 78.8 \pm 2.6$ \\
\midrule[0.005em]
GAT & $81.8 \pm 1.3$ & $71.4 \pm 1.9$ & $78.7 \pm 2.3$ \\
{\bf GraphCON}-GAT & $\bf \color{violet} 83.2 \pm 1.4$ & $\bf \color{violet} 73.2 \pm 1.8$ & $\bf \color{red} 79.5 \pm 1.8$ \\
\midrule[0.005em]
GRAND & $\bf \color{blue} 83.6 \pm 1.0$ & $\bf \color{blue} 73.4 \pm 0.5$ & $\bf \color{violet} 78.8 \pm 1.7$\\
{\bf GraphCON}-Tran & $\bf \color{red} 84.2 \pm 1.3$ & $\bf \color{red}74.2 \pm 1.7$ & $\bf \color{blue}79.4 \pm 1.3$ \\
\bottomrule
\end{tabular}
}
\end{small}
\end{center}
\vskip -0.1in
\end{table}

\textbf{Heterophilic datasets.} We also evaluate GraphCON on the heterophilic graphs; Cornell, Texas and Wisconsin from the WebKB dataset\footnote{http://www.cs.cmu.edu/afs/cs.cmu.edu/project/theo-11/www/wwkb/}. Here, the assumption on neighbor feature similarity does not hold. Many GNN models were shown to struggle in this settings as can be seen by the poor performance of baseline GCN and GAT in \Tref{tab:heterophil}. On the other hand, we see from \Tref{tab:heterophil} that not only do GraphCON-GCN and GraphCON-GAT dramatically outperform the underlying GCN and GAT models (e.g. for the most heterophilic Texas graph, GraphCON-GCN and GraphCON-GAT have mean accuracies of $85.4\%$ and $82.2\%$, compared to accuracies of $55.1\%$ and $52.2\%$ for GCN and GAT), the GraphCON models also provide the best performance, outperforming recent baselines that are specifically designed for heterophilic graphs.

\begin{table}[h]
\caption{Transductive node classification test accuracy (MAP in \%) on heterophilic datasets. All results represent the average performance of the respective model over $10$ fixed train/val/test splits, which are taken from \citet{pei2020geom}. 
}
\label{tab:heterophil}
\vskip 0.15in
\begin{center}
\begin{small}
\resizebox{\columnwidth}{!}{
\begin{tabular}{lccc}
\toprule
 &  \textbf{Texas} & \textbf{Wisconsin} & \textbf{Cornell} \\
 {\em Homophily level}  & $\bf 0.11$ & $\bf 0.21$ & $\bf 0.30$ \\
\midrule
GPRGNN &
$78.4 \pm 4.4$ &
$82.9 \pm 4.2$ &
$80.3 \pm 8.1$ \\

H2GCN &
$\bf \color{blue} 84.9 \pm 7.2$ &
$\bf \color{blue} 87.7 \pm 5.0$ &
$\bf \color{violet} 82.7 \pm 5.3$ \\

GCNII &
$77.6 \pm 3.8$ &
$80.4 \pm 3.4$  &
$77.9 \pm 3.8$ \\

Geom-GCN &
$66.8 \pm 2.7$ &
$64.5 \pm 3.7$ &
$60.5 \pm 3.7$ \\

PairNorm &
$60.3 \pm 4.3$ &
$48.4 \pm 6.1$ &
$58.9 \pm 3.2$ \\

GraphSAGE &
$\bf \color{violet} 82.4 \pm 6.1$ &
$81.2 \pm 5.6$ &
$76.0 \pm 5.0$ \\

MLP &
$80.8 \pm 4.8$ &
$85.3 \pm 3.3$ &
$81.9 \pm 6.4$ \\
\midrule[0.005em]
GAT &
$52.2 \pm 6.6$ &
$49.4 \pm 4.1$ &
$61.9 \pm 5.1$ \\

\textbf{GraphCON}-GAT & 
$82.2 \pm 4.7$ & 
$\bf \color{violet} 85.7 \pm 3.6$ &
$\bf \color{blue} 83.2 \pm 7.0$\\
\midrule[0.005em]
GCN &
$55.1 \pm 5.2$ &
$51.8 \pm 3.1$ &
$60.5 \pm 5.3$ \\

\textbf{GraphCON}-GCN & 
$\bf \color{red} 85.4 \pm 4.2$ & 
$\bf \color{red} 87.8 \pm 3.3$ & 
$\bf \color{red} 84.3 \pm 4.8$\\
         
\bottomrule
\end{tabular}
}
\end{small}
\end{center}
\vskip -0.1in
\end{table}

\subsection{Inductive node classification}
In this experiment, we consider the Protein-Protein-Interaction (PPI) dataset of \citet{ppi}, using the protocol of \citet{graphsage}. 
\Tref{tab:ppi} shows the test performance (micro-average F$1$) of GraphCON and several standard GNN baselines. 
We can see that GraphCON significantly improves the performance of the underling models (GAT from $97.4\%$ to $99.4\%$ and GCN from $98.5\%$ to $99.6\%$, which is the top result on this benchmark). 
\begin{table}[h]
\caption{Test micro-averaged F$_1$ score on Protein-Protein Interactions (PPI) data set. 
}
\label{tab:ppi}
\vskip 0.15in
\begin{center}
\begin{small}
\begin{tabular}{lccc}
\toprule
Model &  Micro-averaged F$_1$ \\
\midrule
VR-GCN \citep{vr-gcn}& $97.8$\\
GraphSAGE \citep{graphsage}& $61.2$\\
PDE-GCN \citep{pde-gcn}& $99.2$ \\
GCNII \citep{gcnii} & $\color{blue} \bf 99.5$ \\
Cluster-GCN \citep{cluster-gcn}& $\color{violet} \bf 99.4$ \\
GeniePath \cite{genie_path} &  $98.5$\\
JKNet \citep{jknet}& $97.6$\\
\midrule[0.005em]
GAT \citep{gat}&  $97.3$\\
{\textbf{GraphCON}}-GAT & $\color{violet} \bf 99.4$\\
\midrule[0.005em]
GCN \citep{gcn} &  $98.5$\\
{\textbf{GraphCON}}-GCN & $\color{red} \bf 99.6$ \\
\bottomrule
\end{tabular}
\end{small}
\end{center}
\vskip -0.1in
\end{table}

\subsection{Molecular graph property regression}
We reproduce the benchmark proposed in \citet{bench_gnns}, regressing the constrained solubulity of 12K molecular graphs from the ZINC dataset \citep{zinc}. 
We follow verbatim the settings of \citet{bench_gnns,DGN}:
make no use of edge features and
constrain the network sizes to $\sim$100K parameters.
\Tref{tab:zinc} summarizes the performance of GraphCON and standard GNN baselines. 
Both GraphCON-GAT and GraphCON-GCN outperform GAT and GCN respectively, 
by a factor of $2$. Moreover, the performance of GraphCON-GCN is on par with the recent state-of-the-art method DGN \citep{DGN} with significantly lower standard deviation. Given these results, it is instructive to ask why GraphCON models outperform their underlying base GNN models such as GCN. A part of the answer can be seen from {\bf SM} \Tref{tab:zinc_depth_vs_loss}, where the MAE for GCN and GraphCON-GCN for this task is shown for increasing number of layers. We observe from this table that while the MAE with GCN \emph{increases} with the number of layers, the MAE for GraphCON-GCN \emph{decreases monotonically} with increasing layers, allowing for the use of very deep GraphCON models with increased expressive power.  

\begin{table}[h]
\caption{Test mean absolute error (MAE, averaged over 4 runs on different initializations) on ZINC (\textbf{without edge features, small 12k version}) restricted to small network sizes of $\sim 100k$ parameters. Baseline results are taken from \citet{DGN}. 
}
\label{tab:zinc}
\vskip 0.15in
\begin{center}
\begin{small}
\begin{tabular}{lccc}
\toprule
Model &  Test MAE \\
\midrule
GIN \citep{gin} & $0.41 \pm 0.008$ \\
GatedGCN \citep{gated_gcn}& $0.42 \pm 0.006$ \\
GraphSAGE \cite{graphsage} & $0.41 \pm 0.005$\\
MoNet \citep{MoNet} & $0.41 \pm 0.007$ \\
PNA \citep{pna}& $\color{violet}\bf 0.32 \pm 0.032$\\
DGN \citep{DGN} & $\color{red}\bf 0.22 \pm 0.010$ \\
\midrule[0.005em] 
GCN \citep{gcn} & $0.47 \pm 0.002$\\
{\textbf{GraphCON}}-GCN & $\color{red}\bf 0.22 \pm 0.004$ \\
\midrule[0.005em]
GAT \citep{gat} &  $0.46 \pm 0.002$\\
{\textbf{GraphCON}}-GAT & $\color{blue}\bf 0.23 \pm 0.004$ \\
\bottomrule
\end{tabular}
\end{small}
\end{center}
\vskip -0.1in
\end{table}

\subsection{MNIST Superpixel graph classification}
This experiment, first suggested by \citet{MoNet}, is based on the MNIST dataset \citep{mnist}, where the grey-scale images are transformed into irregular graphs, as follows: the vertices in the graphs represent superpixels (large blobs of similar color), while the edges represent their spatial adjacency. Each graph has a fixed number of 75 superpixels (vertices). 
We use the standard splitting of using 55K-5K-10K for training, validation, and testing. 

\Tref{tab:mnist} shows  
that GraphCON-GCN dramatically improves the performance of a pure GCN (test accuracy of $88.89\%$ vs $98.70\%$). We stress that both models share the parameters over all layers, i.e. GraphCON-GCN does not have more parameters despite being a deeper model. Thus, the better performance of GraphCON-GCN over GCN can be attributed to the use of more `layers' (iterations) and not to a higher number of parameters (see {\bf SM} \Tref{tab:mnist_depth_vs_acc} for accuracy vs. number of layers for this testcase).
Finally, \Tref{tab:mnist} also shows that GraphCON-GAT outperforms all other methods, including the recently proposed PNCNN \cite{PNCNN}, reaching a nearly-perfect test accuracy of $98.91\%$.

\begin{table}[h]
\caption{Test accuracy in \% on MNIST Superpixel $75$. 
}
\label{tab:mnist}
\vskip 0.15in
\begin{center}
\begin{small}
\begin{tabular}{lccc}
\toprule
Model &  Test accuracy \\
\midrule
ChebNet \citep{chebnet} & $75.62$ \\
MoNet \citep{MoNet}& $91.11$ \\
PNCNN \citep{PNCNN}& $\color{blue}\bf 98.76$\\
SplineCNN \citep{splineCNN} & $95.22$ \\
\midrule[0.005em]
GIN \citep{gin}& $97.23$ \\
{\textbf{GraphCON}}-GIN & $98.53$ \\
\midrule[0.005em]
GatedGCN \citep{gated_gcn}& $97.95$ \\
{\textbf{GraphCON}}-GatedGCN & $98.27$ \\
\midrule[0.005em]
GCN \citep{gcn}& $88.89$\\
{\textbf{GraphCON}}-GCN & $\color{violet}\bf 98.68$\\
\midrule[0.005em]
GAT \citep{gat}& 96.19 \\
{\textbf{GraphCON}}-GAT & $\color{red}\bf 98.91$ \\
\bottomrule
\end{tabular}
\end{small}
\end{center}
\vskip -0.1in
\end{table}
\section{Conclusions}

In conclusion, we proposed a novel framework for designing deep Graph Neural Networks called GraphCON, based on suitable time discretizations of ODEs  \eqref{eq:cont_graphCON} that model the dynamics of a network of controlled and damped oscillators. The coupling between the nodes 
is conditioned on the structure of the underlying graph. 

One can readily interpret GraphCON as a framework to propagate information through multiple layers of a deep GNN, where each hidden layer has the same structure as standard GNNs such as GAT, GCN etc. Unlike in canonical constructions of deep GNNs, which stack hidden layers in a straightforward iterative fashion \eqref{eq:gcss1}, GraphCON stacks them in a more involved manner using the dynamics of the ODE \eqref{eq:cont_graphCON_first_order}. Hence, in principle, any GNN hidden layer can serve as the coupling function $\bF_{\theta}$ in GraphCON \eqref{eq:disc_graphCON}, offering it as an attractive framework for constructing very deep GNNs. 

The well-known oversmoothing problem for GNNs was described mathematically in terms of the stability of zero Dirichlet energy steady states of the underlying ODE  \eqref{eq:cont_graphCON_first_order}. We showed that such zero Dirichlet energy steady states of \eqref{eq:cont_graphCON_first_order}, which lead to constant node features, are not (exponentially) stable. Even if a trajectory reaches a feature vector that is constant across all nodes, very small perturbations will nudge it away and the resulting node features will deviate from each other. Thus, by construction, we demonstrated that the oversmoothing problem, in the sense of definition \ref{def:1}, is mitigated for GraphCON. 

In addition to increasing expressivity by mitigating the oversmoothing problem, GraphCON was rigorously shown to mitigate the exploding and vanishing gradients problem. Consequently, using coupled oscillators also facilitates efficient training of the resulting GNNs. 

Finally, we extensively test GraphCON on a variety of node- and graph-classification and regression tasks, including heterophilic datasets known to be challenging for standard GNN models. 
From these experiments, we observed that (i) GraphCON models significantly outperform the underlying base GNN such as GCN or GAT and (ii) GraphCON models are either on par with or outperform state-of-the-art models on these tasks. This shows that ours is a novel, flexible, easy to use framework for constructing deep GNNs with theoretical guarantees and solid empirical performance.

\newpage
\bibliography{refs}
\bibliographystyle{icml2022}

\newpage
\appendix
\onecolumn
\icmltitlerunning{Supplementary Material for "Graph-Coupled Oscillator Networks"}
\begin{center}
{\bf Supplementary Material for:}\\
Graph-Coupled Oscillator Networks
\end{center}

\section{Further experimental results}
\label{app:further_results}
\subsection{Performance of GraphCON with respect to number of layers}
As we have argued in the main text, GraphCON is designed to be a deep GNN architecture with many layers. Depth could enhance the expressive power of GraphCON and we investigate this issue in three of the datasets, presented in Section \ref{sec:5} of the main text. In the first two experiments, we will focus on the GraphCON-GCN model and compare and contrast its performance, with respect to increasing depth, with the baseline GCN model. 

We start with the molecular graph property regression example for the ZINC dataset of \citet{zinc}. In \Tref{tab:zinc_depth_vs_loss}, we present the mean absolute error (MAE) of the model on the test set with respect to increasing number of layers (up to $20$ layers) of the respective GNNs. As observed from this table, the MAE with standard GCN increases with depth. On the other hand, the MAE with GraphCON decreases as more layers are added.

\begin{table}[h]
\caption{Test mean absolute errors of GraphCON-GCN as well as its baseline model GCN on the ZINC task for different number of layers $N=5,10,15,20$.}
\label{tab:zinc_depth_vs_loss}
\vskip 0.15in
\begin{center}
\begin{small}
\begin{tabular}{lcccc}
\toprule
\multirow{2}{*}{Model} &  \multicolumn{4}{c}{Layers} \\
\cmidrule{2-5}
& 5 & 10 & 15 & 20 \\
\toprule
\textbf{GraphCON}-GCN & $0.241$ & $0.233$ & $0.228$ & $0.214$\\
GCN & $0.442$ & $0.463$ & $0.478$ & $0.489$\\
\bottomrule & & & &
\end{tabular}
\end{small}
\end{center}
\vskip -0.1in
\end{table}

Next, we consider the MNIST Superpixel graph classification task and present the test accuracy with increasing depth (number of layers) for both GCN and GraphCON-GCN. As in the previous example, we observe that increasing depth leads to worsening of the test accuracy for GCN. On the other hand, the test accuracy for GraphCON-GCN increases as more layers (up to $32$ layers) are added to the model.

\begin{table}[h]
\caption{Test accuracies in \% of GraphCON-GCN as well as its baseline model GCN on the MNIST Superpixel 75 task for different number of layers $N=4,8,16,32$.}
\label{tab:mnist_depth_vs_acc}
\vskip 0.15in
\begin{center}
\begin{small}
\begin{tabular}{lcccc}
\toprule
\multirow{2}{*}{Model} &  \multicolumn{4}{c}{Layers} \\
\cmidrule{2-5}
& 4 & 8 & 16 & 32 \\
\toprule
\textbf{GraphCON}-GCN & $97.78$ & $98.51$ & $98.55$ & $98.68$\\
GCN & $88.09$ & $87.26$ & $86.78$ & $85.67$\\
\bottomrule & & & &
\end{tabular}
\end{small}
\end{center}
\vskip -0.1in
\end{table}

Additionally, we compare the performance of GraphCON-GCN to GCN with EdgeDrop \citep{dropedge} (GCN+EdgeDrop), which has been specifically designed to mitigate the oversmoothing phenomenon for deeper GNN models. We consider the Cora node-based classification task in the semi-supervised setting, where we compare GraphCON-GCN to GCN+DropEdge for increasing number of layers $N=2,4,8,16,32,64$. We observe in \Tref{tab:cora_dropedge} that GraphCON
improves (or retains) performance for a large increase in the
number of layers, in contrast to plain GCN+DropEdge on
this task. 
Thus, all three experiments demonstrate that GraphCON leverages more depth to improve performance. 

\begin{table}[h]
\caption{Test accuracies in \% of GraphCON-GCN as well as of GCN+DropEdge on cora (semi-supervised setting) for different number of layers $N=2,4,8,16,32,64$. The GCN+DropEdge results are taken from https://github.com/DropEdge/DropEdge}
\label{tab:cora_dropedge}
\vskip 0.15in
\begin{center}
\begin{small}
\begin{tabular}{lcccccc}
\toprule
\multirow{2}{*}{Model} &  \multicolumn{6}{c}{Layers} \\
\cmidrule{2-7}
& 2 & 4 & 8 & 16 & 32 & 64 \\
\toprule
\textbf{GraphCON}-GCN & $82.20$ & $82.78$ & $83.53$ &
$84.85$ & $82.95$ & $82.12$ \\
GCN+DropEdge & $82.80$ & $82.00$ & $75.80$ & $75.70$ & $62.50$ & $49.50$\\
\bottomrule & & & &
\end{tabular}
\end{small}
\end{center}
\vskip -0.1in
\end{table}

\subsection{Sensitivity of performance of GraphCON to hyperparameters $\alpha$ and $\gamma$}
We recall that GraphCON, \eqref{eq:disc_graphCON} of the main text, has two additional hyperparameters, namely the damping parameter $\alpha \geq 0$ and the frequency control parameter $\gamma > 0$. In Table \ref{tab:best_params}, we present the values of $\alpha,\gamma$ that led to the best performance of the resulting GraphCON models. It is natural to ask how sensitive the performance of GraphCON is to the variation of these hyperparameters. To this end, we choose the MNIST Superpixel graph classification task and perform a sensitivity study of the GraphCON-GCN model with respect to these hyperparameters. First, we fix a value of $\gamma = 0.76$ (corresponding to the best results in Table \ref{tab:best_params}) and vary $\alpha$ in the range of $\alpha \in [0,2]$. The results are plotted in \fref{fig:graphcon_sens_alpha_gamma} and show that the accuracy is extremely robust to a very large parameter range in $\alpha$. Only for large values $\alpha > 1.6$, we see that the accuracy deteriorates when the damping is too high. 

Next for this model and task, we fix $\alpha = 1$ (which provides the best performance as reported in Table \ref{tab:best_params}) and vary $\gamma \in [0,2]$. Again, for a large range of values corresponding to $\gamma \in [0.2,2]$, the accuracy is very robust. However, for very small values of $\gamma$, the accuracy falls significantly. This is to be expected as the model loses its interpretation as system of oscillators for $\gamma \approx 0$. 

Thus, these sensitivity results demonstrate that GraphCON performs very robustly with respect to variations of the parameters $\alpha,\gamma$, within a reasonable range. 

\begin{figure}[h]
\vskip 0.2in
\begin{center}
\centerline{\includegraphics[width=0.5\columnwidth]{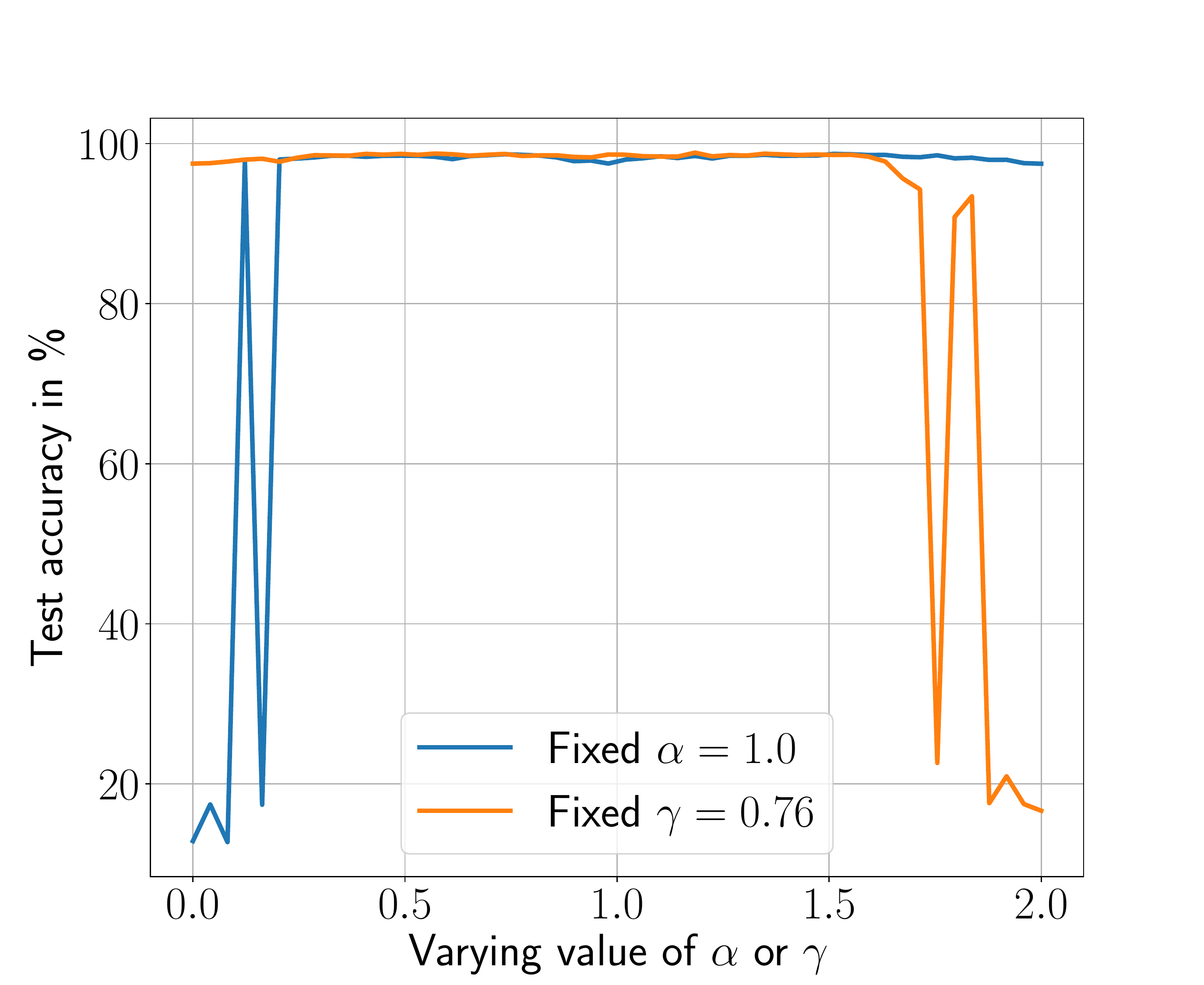}}
\caption{Sensitivity (measured as test accuracy) plot for $\alpha$ and $\gamma$ hyperparameters of GraphCON-GCN (with $32$ layers) trained on MNIST superpixel 75 experiment. First, $\alpha=1.0$ is fixed and $\gamma$ is varied in $[0,2]$. Second, $\gamma=0.76$ is fixed and $\alpha$ is varied in $[0,2]$. The fixed $\alpha,\gamma$ are taken from the best performing GraphCON-GCN on the MNIST superpixel 75 task (\Tref{tab:best_params})}
\label{fig:graphcon_sens_alpha_gamma}
\end{center}
\vskip -0.2in
\end{figure}

\section{Training details}
\label{app:train_details}
All experiments were run on NVIDIA GeForce GTX 1080 Ti, RTX 2080 Ti as well as RTX 2080 Ti GPUs. The tuning of the hyperparameters was done using a standard random search algorithm.
We fix the time-step $\Dt$ in \eqref{eq:disc_graphCON} to $1$ in all experiments. The damping parameter $\alpha$ as well as the frequency control parameter $\gamma$ are set to $1$ for all Cora, Citeseer and Pubmed experiments, while we set them to $0$ for all experiments based on the Texas, Cornell and Wisconsin network graphs.
For all other experiments we include $\alpha$ and $\gamma$ to the hyperparameter search-space. The tuned values can be found in \Tref{tab:best_params}.

\begin{table}[h]
\caption{Hyperparameters $\alpha$ and $\gamma$ of GraphCON \eqref{eq:disc_graphCON} for each best performing GraphCON model (based on a validation set).}
\label{tab:best_params}
\vskip 0.15in
\begin{center}
\begin{small}
\begin{tabular}{llccccc}
\toprule
Model & \textbf{Experiment} & $\alpha$ & $\gamma$ \\
\midrule
GraphCON-GCN & \multirow{2}{*}{\textbf{PPI}} & $0.242$ & $1.0$ \\
GraphCON-GAT & & $0.785$ & $1.0$\\
\midrule[0.005em]

GraphCON-GCN & \multirow{2}{*}{\textbf{ZINC}} & $0.215$ & $1.115$ \\
GraphCON-GAT & & $1.475$ & $1.324$ \\
\midrule[0.005em]

GraphCON-GCN & \multirow{2}{*}{\textbf{MNIST (superpixel)}} &  $1.0$ & $0.76$ \\
GraphCON-GAT & &  $0.76$ & $0.105$ \\
\bottomrule
\end{tabular}
\end{small}
\end{center}
\vskip -0.1in
\end{table}

\section{Mathematical details for Section \ref{sec:3} of main text}
\label{app:pf}
In this section, we provide details for the mathematical results in section \ref{sec:3} of the main text. We start with,
\subsection{Proof of Proposition \ref{prop:1}}
\label{app:pf1}
\begin{proof}
We multiply $\bY_i^\top$ to the second equation of \eqref{eq:ODEnw1} and obtain,
\begin{align*}
    \bY_i^\top \frac{d\bY_i}{dt} &= \sum\limits_{j \in \cN_i} \bA_{ij}\bY_i^\top \left(\bX_j - \bX_i\right), \quad \left({\rm as}~
    \sum\limits_{j\in\cN_i} \bA_{ij}= 1\right)
\end{align*}
Summing over $i \in \mathcal{V}$ and using the symmetry condition $\bA_{ij} = \bA_{ji}$ in the above expression yields,
\begin{align*}
    \frac{d}{dt}\sum\limits_{i\in \mathcal{V}} \frac{\|\bY_i\|^2}{2} &= -\sum\limits_{i \in \mathcal{V}} \sum\limits_{j \in \cN_i} \bA_{ij}\left(\bY_j - \bY_i\right)^\top \left(\bX_j - \bX_i\right), \\
    &=- \sum\limits_{i \in \mathcal{V}} \sum\limits_{j \in \cN_i} \bA_{ij}\left(\frac{d(\bX_j - \bX_i)}{dt}\right)^\top \left(\bX_j - \bX_i\right) \\
    \Rightarrow &\frac{1}{2}\frac{d}{dt}\left( \sum\limits_{i\in \mathcal{V}} \|\bY_i\|^2 + \sum\limits_{i \in \mathcal{V}} \sum\limits_{j \in \cN_i} \bA_{ij} \|\bX_j-\bX_i\|^2\right) = 0.
\end{align*}
Integrating the last line in the above expression over time $[0,t]$ yields the desired identity \eqref{eq:ebal}

\end{proof}

\subsection{Proof of Proposition \ref{prop:2}}
\label{app:pf2}
\begin{proof}

By the definition of the Dirichlet energy \eqref{eq:graph_H1}, \eqref{eq:vgg1} implies that,
\begin{equation}
    \label{eq:vgg2}
    \lim\limits_{t \rightarrow \infty} \bX_i (t) \equiv \bc, \quad \forall i \in \cV,
\end{equation}
for some $\bc \in \R^{{m}}$. In other words, all the hidden node features converge to the same feature vector $\bc$ as time increases. Moreover, by \eqref{eq:vgg1}, this convergence is exponentially fast. 

Plugging in \eqref{eq:vgg2} in to the first equation of the ODE \eqref{eq:cont_graphCON_first_order}, we obtain that,
\begin{equation}
    \label{eq:vgg3}
    \lim\limits_{t \rightarrow \infty} \bY_i (t) \equiv {\bf 0}, \quad \forall i \in \cG,
\end{equation}
with ${\bf 0}$ being the ${m}$ vector with zeroes for all its entries. Thus, oversmoothing in the sense of definition \ref{def:1}, amounts to $\left(\bc,{\bf 0}\right)$ being an exponentially stable fixed point (steady state) for the dynamics of \eqref{eq:ODEnw1}
    
On the other hand, if $\left(\bc, {\bf 0}\right)$ is an exponentially stable steady state of \eqref{eq:ODEnw1}, then the trajectories converge to this state exponentially fast satisfying \eqref{eq:vgg1}. Consequently, by the definition of the 
Dirichlet energy \eqref{eq:graph_H1}, we readily observe that the oversmoothing problem, in the sense of definition \ref{def:1}, occurs in this case.   

\end{proof}
\subsection{Proof of Proposition \ref{prop:3}}
\label{app:pf3}
The main aim of the section is to show that steady states of \eqref{eq:ODEnw1}, of the form $\left(\bc,{\bf 0}\right)$ are not exponentially stable.  

To this end, we fix $\bc$ and start by considering small perturbations around the fixed point $\left(\bc,{\bf 0}\right)$. We define, 
$$
\hat{\bX}_i = \bX_i - \bc, \hat{\bY}_i = \bY_i,
$$
and evolve these perturbations by the linearized ODE, 
\begin{equation}
\label{eq:ODEl1}
\begin{aligned}
\hat{\bX}_i^{\prime} &= \hat{\bY}_i, \\
\hat{\bY_i}^{\prime} &= \sigma^\prime(\bc) \sum\limits_{j \in \cN_i} \hat{\bA}_{i,j} \hat{\bX}_j  - \hat{\bX}_i - \alpha \hat{\bY}_i, \end{aligned}
\end{equation}
As $\sigma(x) = \max(x,0)$ and $\bc \geq 0$, we have that $\sigma^{\prime}(\bc) = ID$ and linearized system \eqref{eq:ODEl} reduces to, 
\begin{equation}
\label{eq:ODEl}
\begin{aligned}
\hat{\bX}_i^{\prime} &= \hat{\bY}_i, \\
\hat{\bY_i}^{\prime} &=  \sum\limits_{j \in \cN_i} \hat{\bA}_{ij} \hat{\bX}_j  - \hat{\bX}_i - \alpha \hat{\bY}_i, \end{aligned}
\end{equation}
with 
\begin{equation}
\label{eq:ahij}
\begin{aligned}
\hat{\bA}_{ij} &= \bA_{ij}(\bc,\bc), \quad \forall j \in \cN_i, \quad \forall i \in \cG, \\
0 &\leq \hat{A}_{ij} \leq 1, \quad \sum\limits_{j\in \cN_i} \hat{A}_{ij} = 1.
\end{aligned}
\end{equation}
We have the following proposition on the dynamics of linearized system \eqref{eq:ODEl} with respect to perturbations of the fixed point $\left(\bc,{\bf 0}\right)$,
\begin{proposition}
\label{prop:4}
Perturbations $\hat{\bX}(t),\hat{\bY}(t)$ of the fixed point $\left(\bc,{\bf 0}\right)$, which evolve according to \eqref{eq:ODEl} satisfy the following identity, 
\begin{equation}
    \label{eq:ebal1}
    \begin{aligned}
   &\frac{1}{v} \left(\sum\limits_{i \in \cV} \|\hat{\bY_i(t)}\|^2 + \sum\limits_{i \in \cV}\sum\limits_{j\in \cN_i} \frac{\hat{\bA}_{ij}+ \hat{\bA}_{ji}}{2}\left( \|\hat{\bX}_j(t)-\hat{\bX}_i(t)\|^2\right)\right)
    = T_1(t) + T_2(t) + T_3(t), \\
T_1(t) &=\frac{1}{v}\sum\limits_{i \in \cV} \left(\|\hat{\bY_i(0)}\|^2 \right)e^{-2\alpha t}
    + \frac{1}{v}\sum\limits_{i \in \cV}\sum\limits_{j\in \cN_i} \frac{\hat{\bA}_{ij}+ \hat{\bA}_{ji}}{2}\left( \|\hat{\bX}_j(0)-\hat{\bX}_i(0)\|^2 \right) e^{-2\alpha t} \\
   T_2(t) &=  \frac{\alpha}{v} \sum\limits_{i \in \cV} \sum\limits_{j\in \cN_i} \left(\hat{\bA}_{ij}+ \hat{\bA}_{ji}\right) \int\limits_0^t \|\hat{\bX}_j(s)-\hat{\bX}_i(s)\|^2e^{2\alpha (s-t)} ds \\
  T_3(t)  &= \frac{1}{v}\sum\limits_{i \in \cV} \sum\limits_{j\in \cN_i} \left(\hat{\bA}_{ij} - \hat{\bA}_{ji}\right) \int\limits_0^t \left(\hat{\bY}_i(s)+\hat{\bY}_j(s)\right) ^\top \left(\hat{\bX}_j(s)-\hat{\bX}_i(s) \right)e^{2\alpha (s-t)} ds 
    \end{aligned}
\end{equation}
\end{proposition}
\begin{proof}
Multiplying the second equation in \eqref{eq:ODEl} with $\hat{\bY}_i^\top$ and using the fact that $\sum\limits_{j\in \cN_i} \hat{\bA}_{ij} = 1$, we obtain,
\begin{equation}
    \label{eq:prf1}
    \begin{aligned}
    &\frac{d}{dt}\frac{\|\hat{\bY}_i\|^2}{2} + \alpha \|\hat{\bY}_i\|^2 = \sum\limits_{j\in \cN_i} \hat{\bA}_{ij} \hat{\bY}_i^\top \left(\hat{\bX}_j - \hat{\bX}_i \right), \\
    &= \sum\limits_{j\in \cN_i} \hat{\bA}_{ij} \frac{\left(\hat{\bY}_i+\hat{\bY}_j\right)^\top}{2} \left(\hat{\bX}_j - \hat{\bX}_i \right) 
    -\sum\limits_{j\in \cN_i} \hat{\bA}_{ij} \frac{\left(\hat{\bY}_j-\hat{\bY}_i\right)^\top}{2} \left(\hat{\bX}_j - \hat{\bX}_i \right), \\
    &= \sum\limits_{j\in \cN_i} \hat{\bA}_{ij} \frac{\left(\hat{\bY}_i+\hat{\bY}_j\right)^\top}{2} \left(\hat{\bX}_j - \hat{\bX}_i \right) 
    -\sum\limits_{j\in \cN_i} \frac{\hat{\bA}_{ij}}{2} \frac{d}{dt}\left(\hat{\bX}_j-\hat{\bX}_i\right)^\top \left(\hat{\bX}_j - \hat{\bX}_i \right),
    \end{aligned}
\end{equation}
where we have used the first equation of \eqref{eq:ODEl} in the last line of \eqref{eq:prf1}. Consequently, we have for all $i \in \cV$,
\begin{equation}
    \label{eq:prf2}
    \begin{aligned}
    \frac{d}{dt}\frac{\|\hat{\bY}_i\|^2}{2} &+ \alpha \|\hat{\bY}_i\|^2
    + \frac{d}{dt} \sum\limits_{j\in \cN_i} \frac{\hat{\bA}_{ij}}{2} \frac{\|\hat{\bX}_j-\hat{\bX}_i\|^2}{2} \\
    &= \sum\limits_{j\in \cN_i} \hat{\bA}_{ij} \frac{\left(\hat{\bY}_i+\hat{\bY}_j\right)^\top}{2} \left(\hat{\bX}_j - \hat{\bX}_i \right)
 \end{aligned}
\end{equation}
Summing \eqref{eq:prf2} over all nodes $i \in \cV$ yields, 
\begin{equation}
    \label{eq:prf3}
    \begin{aligned}
    \frac{d}{dt}\sum\limits_{i \in \cV} \frac{\|\hat{\bY}_i\|^2}{2} &+ \alpha \sum\limits_{i \in \cV} \|\hat{\bY}_i\|^2 
    + \frac{d}{dt} \sum\limits_{i \in \cV} \sum\limits_{j\in \cN_i} \frac{\hat{\bA}_{ij}+ \hat{\bA}_{ji}}{2} \frac{\|\hat{\bX}_j-\hat{\bX}_i\|^2}{2} \\
    &= \sum\limits_{i \in \cV} \sum\limits_{j\in \cN_i} \frac{\hat{\bA}_{ij} - \hat{\bA}_{ji}}{2} \left(\hat{\bY}_i+\hat{\bY}_j\right)^\top \left(\hat{\bX}_j - \hat{\bX}_i \right)
 \end{aligned}
\end{equation}
Multiplying $e^{2\alpha t}$ to both sides of \eqref{eq:prf3} and using the chain rule, we readily obtain, 
\begin{equation}
    \label{eq:prf30}
    \begin{aligned}
    &\frac{d}{dt}\sum\limits_{i \in \cV} e^{2\alpha t} \left(\frac{\|\hat{\bY}_i\|^2}{2} + \sum\limits_{j\in \cN_i} \frac{\hat{\bA}_{ij}+ \hat{\bA}_{ji}}{2} \frac{\|\hat{\bX}_j-\hat{\bX}_i\|^2}{2} \right) \\
    &= \alpha e^{2\alpha t} \sum\limits_{i \in \cV} \sum\limits_{j\in \cN_i} \frac{\hat{\bA}_{ij}+ \hat{\bA}_{ji}}{2} \|\hat{\bX}_j-\hat{\bX}_i\|^2 \\
    &+ e^{2\alpha t} \sum\limits_{i \in \cV} \sum\limits_{j\in \cN_i} \frac{\hat{\bA}_{ij} - \hat{\bA}_{ji}}{2} \left(\hat{\bY}_i+\hat{\bY}_j\right)^\top \left(\hat{\bX}_j - \hat{\bX}_i \right)
    \end{aligned}
    \end{equation}
Integrating \eqref{eq:prf30} over the time interval $[0,t]$ yields, 
\begin{equation}
    \label{eq:prf4}
    \begin{aligned}
&\sum\limits_{i \in \cV} \left(\frac{\|\hat{\bY_i(t)}\|^2}{2} \right) e^{2\alpha t} + \sum\limits_{i \in \cV}\sum\limits_{j\in \cN_i} \frac{\hat{\bA}_{ij}+ \hat{\bA}_{ji}}{2}\left( \frac{\|\hat{\bX}_j(t)-\hat{\bX}_i(t)\|^2}{2} \right)e^{2\alpha t} \\
&= \sum\limits_{i \in \cV} \left(\frac{\|\hat{\bY_i(0)}\|^2}{2} \right)
+ \sum\limits_{i \in \cV}\sum\limits_{j\in \cN_i} \frac{\hat{\bA}_{ij}+ \hat{\bA}_{ji}}{2}\left( \frac{\|\hat{\bX}_j(0)-\hat{\bX}_i(0)\|^2}{2} \right) \\
&+ \alpha \sum\limits_{i \in \cV} \sum\limits_{j\in \cN_i} \frac{\hat{\bA}_{ij}+ \hat{\bA}_{ji}}{2} \int\limits_0^t \|\hat{\bX}_j(s)-\hat{\bX}_i(s)\|^2e^{2\alpha s} ds \\
&+  \sum\limits_{i \in \cV} \sum\limits_{j\in \cN_i} \frac{\hat{\bA}_{ij} - \hat{\bA}_{ji}}{2} \int\limits_0^t \left(\hat{\bY}_i(s)+\hat{\bY}_j(s)\right)^\top \left(\hat{\bX}_j(s) - \hat{\bX}_i(s) \right) e^{2\alpha s} ds  
  \end{aligned}
    \end{equation}
We readily obtain the desired identity \eqref{eq:ebal1} from \eqref{eq:prf4}.
\end{proof}

Next, we observe that the right-hand side of the nonlinear ODEs \eqref{eq:ODEnw1} is \emph{globally Lipschitz}. Therefore, solutions exist for all time $t > 0$, are unique and depend continuously on the data. 

We assume that the initial perturbations around the steady state $\left(\bc,{\bf 0}\right)$ are small i.e., they satisfy  
\begin{align*}
\|\hat{\bX}_i(0) - \hat{\bX}_j(0)\| &\leq  \ep, \quad \forall j \in \cN_i, \quad \forall i \in \cV, \\
\|\hat{\bY}_i(0)\| &\leq \ep, \quad \forall i \in \cV, 
\end{align*}
for some $0 < \ep << 1$.

Hence, there exists a small time $\tau > 0$ such that the time-evolution of these perturbations can be approximated to arbitrary accuracy by solutions of the linearized system \eqref{eq:ODEl}. 

Next, we see from the identity \eqref{eq:ebal1} that the evolution of the perturbations $\hat{\bX},\hat{\bY}$ from the fixed point $\left(\bc,{\bf 0}\right)$ for the linearized system \eqref{eq:ODEl} is balanced by three terms $T_{1,2,3}$. The term $T_1$ is clearly a \emph{dissipative} term and says that the initial perturbations are damped exponentially fast in time. 

On the other hand, the term $T_2$, which has a positive sign, is a \emph{production} term and says that the initial perturbations will \emph{grow} with time $t$. Given the continuous dependence of the dynamics evolved by the ODE \eqref{eq:ODEl}, there exists a time, still called $\tau $ by choosing it even smaller than the $\tau$ encountered before, such that
\begin{equation}
\label{eq:prf6}
\begin{aligned}
\|\hat{\bX}_i(t) - \hat{\bX}_j(t)\| &\sim \ord(\ep), \quad \forall j \in \cN_i, \quad \forall i \in \cV, \quad \forall t \in [0,\tau], \\
\|\hat{\bY}_i(t)\| &\sim \ord(\ep), \quad \forall i \in \cV, \forall t \in [0,\tau].
\end{aligned}
\end{equation}
Plugging the above expression into the term $T_2$ in \eqref{eq:ebal1} and using the right-stochasticity of the matrix $\hat{\bA}$, we obtain that,
\begin{equation}
    \label{eq:prf5}
    T_2(t) \sim \ord(\ep^2)\left(1 - e^{-2\alpha t}\right), \quad \forall t \leq \tau 
\end{equation}
Thus, the leading term in $T_2$ \emph{grows} algebraically with respect to the initial perturbations.

Next we turn our attention to the term $T_3$ in \eqref{eq:ebal1}. This term is proportional to the \emph{asymmetry} in the graph-coupling matrix $\hat{\bA} = \bA(\bc,\bc)$. If this matrix were symmetric, then $T_3$ vanishes. On the other hand, for many $1$-neighborhood couplings considered in this article, the matrix $\hat{\bA}$ is not symmetric. In fact, one can explicitly compute that for the GAT and Transformers attention and GCN-couplings, we have,
\begin{equation}
    \label{ahatij}
    \hat{\bA}_{ij} = \frac{1}{{\rm deg}(i)}, \quad \forall j \in \cN_i, \quad \forall i \in \cV.
\end{equation}
Here, ${\rm deg}$ refers to the degree of the node, with possibly inserted self-loops. 

As the ordering of nodes of the graph $\cG$ is arbitrary, we can order them in such a manner that $\hat{\bA}_{ij} > \hat{\bA}_{ji}$. Even with this ordering, as long as the matrix $\hat{\bA}$ is not symmetric, the term $T_3$ is of \emph{indefinite sign}. If it is positive, then we have additional growth with respect to time in \eqref{eq:ebal1}. On the other hand, if $T_3$ is negative, it will have a \emph{dissipative effect}. The rate of this dissipation can be readily calculated for a short time $t \leq \tau$ under the assumption \eqref{eq:prf6} to be, 
\begin{equation}
\label{eq:prf7}
|T_3(t)| \sim \frac{\overline{D}-\underline{D}}{\overline{D}\underline{D}}\left(\frac{1-e^{-2\alpha t}}{2\alpha}\right) \ord(\ep^2).
\end{equation}
Here, we define,
\begin{equation}
    \label{eq:deg}
    \overline{D} = \max\limits_{i \in \cV} {\rm deg}(i), \quad \underline{D} = \min\limits_{i \in \cV} {\rm deg}(i)
\end{equation}
Thus by combining \eqref{eq:prf5} with \eqref{eq:prf7}, we obtain,
\begin{equation}
    \label{eq:prf8}
    T_2 + T_3 \sim \left(1 -\frac{\overline{D}-\underline{D}}{2\alpha\overline{D}\underline{D}}\right) \left(1-e^{-2\alpha t}\right) \ord(\ep^2)
\end{equation}
In particular for $\alpha \geq 1/2$, we see from \eqref{eq:prf8}, that the overall balance \eqref{eq:ebal1} leads to an algebraic growth, rather than exponential decay, of the initial perturbations of the fixed point $(\bc,{\bf 0}) $. Thus, we have shown that this steady state is not exponentially stable and small perturbations will take the trajectories of the ODE \eqref{eq:ODEl} away from this fixed point, completing the proof of Proposition \ref{prop:3}. 

\begin{remark}
We see from the above proof, the condition $\alpha \geq \frac{1}{2}$ is only a sufficient condition for the proof of Proposition \ref{prop:3}, we can readily replace it by,
$$
\alpha \geq \frac{\overline{D}-\underline{D}}{2\overline{D}\underline{D}}
$$

\end{remark}
\pagebreak

\subsection{Proofs of Propositions \ref{prop:gub} and \ref{prop:glb}}
\label{app:evgp}
As a first step in proving the gradient bounds in Proposition \ref{prop:gub}, we will prove the following upper bound on the hidden node features of the following general form of GraphCON \eqref{eq:disc_graphCON}, written node-wise as,
\begin{equation}
    \label{eq:GCnw}
    \begin{aligned}
    \bC^{n-1}_i &= (\bF_\theta(\bX^{n-1}))_i, \\ 
    \bY^n_i &= \bY^{n-1}_i + \Dt \sigma(\bC^{n-1}_i) - \gamma \Dt \bX^{n-1}_i - \alpha \Dt \bY^{n-1}_i, \\
    \bX^n_i &= \bX^{n-1}_i + \Dt \bY^n_i.
    \end{aligned}
\end{equation}
We derive following upper bound on the resulting hidden node features,
\begin{proposition}
\label{prop:b2}
For all n, let $t_n = n\Dt$ and the time step $\Dt$ satisfy,
\begin{align*}
    \Dt < \min \left(\frac{\alpha}{\gamma},\frac{1}{\alpha}\right)
\end{align*}
Let $\bX^n_i$ denote the hidden state vector at any node $i \in \cV$ which evolves according to GraphCON \eqref{eq:GCnw}, then the hidden states satisfy the following bound,
\begin{equation}
    \label{eq:ebd2}
    \begin{aligned}
    \|\bX^n_i\|^2 &\leq \|\bX^0_i\|^2 +  \frac{1}{\gamma} \|\bY^0_i\|^2 \\
    &+ \frac{{m} \beta^2t_n}{2\gamma(\alpha - \gamma \Dt)}
    \end{aligned}
\end{equation}
where $\beta$ is the global bound on the underlying activation function $\sigma$ \eqref{eq:def1}.
\end{proposition}
\begin{proof}
We multiply $\gamma (\bX^{n-1}_i)^\top$ to the third equation of \eqref{eq:GCnw} and $(\bY^n_i)^\top$ to the second equation of \eqref{eq:GCnw} and repeatedly use the following elementary identities, 
\begin{align*}
{\bf a}^\top({\bf a}-\bb) &= \frac{\|{\bf a}\|^2}{2} - \frac{\|\bb\|^2}{2} + \frac{1}{2}\|{\bf a}-\bb\|^2, \\
\bb^\top({\bf a}-\bb) &= \frac{\|{\bf a}\|^2}{2} - \frac{\|\bb\|^2}{2} - \frac{1}{2}\|{\bf a}-\bb\|^2,
\end{align*}
to obtain,
\begin{align*}
    \gamma \frac{\|\bX^n_i\|^2}{2} + \frac{\|\bY^n_i\|^2}{2} &= 
    \gamma \frac{\|\bX^{n-1}_i\|^2}{2} + \frac{\|\bY^{n-1}_i\|^2}{2} \\
    &+ \Dt (\bY_i^n)^\top \sigma(\bC^{n-1}_i) \\
    &+ \Dt \left(\frac{\gamma \Dt}{2} - \alpha  + \frac{\alpha}{2}\right)  \|\bY^n_i\|^2 \\
    &- \frac{\alpha \Dt}{2} \|\bY^{n-1}_i\|^2 \\
    &+ \left(\frac{\alpha \Dt -1}{2}\right) \|\bY^n_i - \bY^{n-1}_{i}\|^2
\end{align*}
As we have assumed that the time step $\Dt$ is chosen such that 
\begin{align*}
    \Dt < \min \left(\frac{\alpha}{\gamma},\frac{1}{\alpha}\right)
\end{align*}
we obtain from the above inequality that, 
\begin{align*}
    \gamma \frac{\|\bX^n_i\|^2}{2} + \frac{\|\bY^n_i\|^2}{2} &\leq 
    \gamma \frac{\|\bX^{n-1}_i\|^2}{2} + \frac{\|\bY^{n-1}_i\|^2}{2} \\
    &+ \Dt (\bY_i^n)^\top \sigma(\bC^{n-1}_i) \\
    &- \Dt \left(\frac{\alpha - \gamma \Dt}{2} \right)  \|\bY^n_i\|^2
\end{align*}
Next we use the elementary identity
\begin{align*}
    {\bf a}^\top \bb \leq \frac{\epsilon \|{\bf a}\|^2}{2} + \frac{\|\bb\|^2}{2\epsilon},
\end{align*}
with $\epsilon = \alpha-\gamma \Dt$ in the above inequality to obtain,
\begin{equation}
    \label{eq:pf101}
    \begin{aligned}
    \gamma \frac{\|\bX^n_i\|^2}{2} + \frac{\|\bY^n_i\|^2}{2} &\leq 
    \gamma \frac{\|\bX^{n-1}_i\|^2}{2} + \frac{\|\bY^{n-1}_i\|^2}{2} \\
    &+ \frac{\Dt}{2(\alpha - \gamma \Dt)}\|\sigma(\bC^{n-1}_i)\|^2
    \end{aligned}
\end{equation}
Now from the bound \eqref{eq:def1} on the activation function, we obtain from \eqref{eq:pf101} that,
\begin{equation}
    \label{eq:pf102}
    \begin{aligned}
    \gamma \frac{\|\bX^n_i\|^2}{2} + \frac{\|\bY^n_i\|^2}{2} &\leq 
    \gamma \frac{\|\bX^{n-1}_i\|^2}{2} + \frac{\|\bY^{n-1}_i\|^2}{2} \\
    &+ \frac{{m} \Dt\beta^2}{2(\alpha - \gamma \Dt)}
    \end{aligned}
\end{equation}
Iterating \eqref{eq:pf102} over $n$ yields,
\begin{equation}
    \label{eq:pf1020}
    \begin{aligned}
    \gamma \|\bX^n_i\|^2 + \|\bY^n_i\|^2 &\leq 
    \gamma \|\bX^{0}_i\|^2 + \|\bY^{0}_i\|^2 \\
    &+ \frac{{m} n \Dt\beta^2}{2(\alpha - \gamma \Dt)},
    \end{aligned}
\end{equation}
which readily yields the desired inequality \eqref{eq:ebd2}.
\end{proof}

\subsubsection{Proof of Proposition \ref{prop:gub}}
\begin{proof}
For any $\ell \leq n \leq N$, a tedious yet straightforward computation yields the following representation formula,
\begin{equation}
    \label{eq:grad6}
    \frac{\partial \bZ^n}{\partial \bZ^{n-1}} = \ident_{2v \times 2v} + \Dt \bE^{n,n-1} + \Dt^2 \bF^{n,n-1}.
\end{equation}
Here $\bE^{n,n-1} \in \R^{2v \times 2v}$ is a matrix whose entries are given below. For any $1 \leq i \leq v$, we have,
\begin{align*}
    \bE^{n,n-1}_{2i-1,2i} &= 1, \\
    \bE^{n,n-1}_{2i-1,j} &=0, \quad \forall j \neq 2i, \\
     \bE^{n,n-1}_{2i,2i} &= -1, \\
      \bE^{n,n-1}_{2i,2i-1} &= -1 +  \frac{\sigma^{\prime}(\bC^{n-1}_i) \bw^n_i}{d_i}, \\
       \bE^{n,n-1}_{2i,2j} &= 0, \quad \forall 1 \leq j \leq v, ~{\rm and}~ j\neq i, \\
       \bE^{n,n-1}_{2i,2j-1} &= \frac{\sigma^{\prime}(\bC^{n-1}_j)\bw^n_j}{\sqrt{d_id_j}}, \forall j \in \cN_i, \\
        \bE^{n,n-1}_{2i,2j-1} &= 0, \quad \forall j \notin \cN_i ~ {\rm and}~ j \neq i.
\end{align*}
Similarly, 
$\bF^{n,n-1} \in \R^{2v \times 2v}$ is a matrix whose entries are given below. For any $1 \leq i \leq v$, we have,
\begin{align*}
    \bF^{n,n-1}_{2i,j} &= 0, \quad \forall  j,\\
      \bF^{n,n-1}_{2i-1,2i-1} &= -1 +  \frac{\sigma^{\prime}(\bC^{n-1}_i) \bw^n_i}{d_i}, \\
       \bF^{n,n-1}_{2i-1,2j-1} &= \frac{\sigma^{\prime}(\bC^{n-1}_j)\bw^n_j}{\sqrt{d_id_j}}, \forall j \in \cN_i, \\
        \bF^{n,n-1}_{2i-1,2j-1} &= 0, \quad \forall j \notin \cN_i ~ {\rm and}~ j \neq i.
\end{align*}
Using \eqref{eq:def1}, 
it is straightforward to compute that,
\begin{equation}
    \label{eq:gd1}
    \begin{aligned}
    \|\bE^{n,n-1}\|_{\infty} &\leq 2 + \beta^{\prime}\hat{D}\|\bw^n\|_1, \\
     \|\bF^{n,n-1}\|_{\infty} &\leq  1 + \beta^{\prime}\hat{D}\|\bw^n\|_1,
    \end{aligned}
\end{equation}
Then using $\Dt \leq 1$ and definition \eqref{eq:def1}, we have from \eqref{eq:grad6} that, 
\begin{align*}
   \left \|\frac{\partial \bZ^n}{\partial \bZ^{n-1}}\right\|_{\infty} \leq 1 + \frac{\Gamma}{2}\Dt, \quad \forall n.
\end{align*}
Therefore, from the identity \eqref{eq:prod}, we obtain,
\begin{align*}
   \left \|\frac{\partial \bZ^N}{\partial \bZ^{\ell}}\right\|_{\infty} \leq \left(1 + \frac{\Gamma}{2}\Dt\right)^{N-\ell}.
\end{align*}
Now choosing $\Dt << 1$ small enough such that the following inequality holds,
\begin{equation}
    \label{eq:gd2}
    \left(1 + \frac{\Gamma}{2}\Dt\right)^{N-\ell} \leq 1 + (N-\ell) \Gamma \Dt,
\end{equation}
leads to the following bound,
\begin{equation}
    \label{eq:gd3}
   \left \|\frac{\partial \bZ^N}{\partial \bZ^{\ell}}\right\|_{\infty} \leq 1 + (N-\ell) \Gamma \Dt \leq 1 + N\Gamma\Dt
\end{equation}
A straight-forward differentiation of the loss function \eqref{eq:lf} yields,
\begin{equation}
\label{eq:gd4}
\frac{\partial \cJ}{\partial \bZ^N} = \frac{1}{v}\left[\bX^N_1-\overline{\bX}_1,0,\bX^N_2-\overline{\bX}_2,0,\cdots,\bX^N_v-\overline{\bX}_v,0\right].
\end{equation}
Hence, 
\begin{equation}
    \label{eq:gd5}
    \left\|\frac{\partial \cJ}{\partial \bZ^N}\right\|_{\infty} \leq \frac{1}{v}\left(\max\limits_{1 \leq i \leq v} |\bX^N_i| +  \max\limits_{1 \leq i \leq v} |\overline{\bX}_i| \right)
\end{equation}
Applying the pointwise upper bound \eqref{eq:ebd2} to \eqref{eq:gd5}, we obtain, 
\begin{equation}
    \label{eq:gd6}
    \left\|\frac{\partial \cJ}{\partial \bZ^N}\right\|_{\infty} \leq \frac{1}{v}\left(\max\limits_{1 \leq i \leq v} (|\bX^0_i| + |\bY^0_i|) + \max\limits_{1 \leq i \leq v} |\overline{\bX}_i| + \beta \sqrt{N\Dt}\right)
\end{equation}
Finally, a direct calculation provides the following characterization of the vector $\frac{\partial \bZ^\ell}{\partial \bw^\ell_k} \in \R^{2v}$,
\begin{equation}
    \label{eq:gd7}
    \begin{aligned}
    \left(\frac{\partial \bZ^\ell}{\partial \bw^\ell_k}\right)_{2j} &= \Dt \frac{\sigma^{\prime}(\bC^{\ell}_j)\bX^{\ell-1}_j}{\sqrt{d_kd_j}}, \quad j \in \cN_k, \\
 \left(\frac{\partial \bZ^\ell}{\partial \bw^\ell_k}\right)_{2j-1}   &= \Dt^2 \frac{\sigma^{\prime}(\bC^{\ell}_j)\bX^{\ell-1}_j}{\sqrt{d_kd_j}}, \quad j \in \cN_k, \\
 \left(\frac{\partial \bZ^\ell}{\partial \bw^\ell_k}\right)_{j} &\equiv 0, \quad {\rm otherwise}.
    \end{aligned}
\end{equation}
Therefore using the pointwise bound \eqref{eq:ebd2}, one can readily calculate that, 
\begin{equation}
    \label{eq:gd8}
    \begin{aligned}
    \left\|\frac{\partial \bZ^\ell}{\partial \bw^\ell_k}\right\|_{\infty} &\leq \Dt \beta^{\prime} \hat{D} \left(\max\limits_{1 \leq i \leq v} (|\bX^0_i| + |\bY^0_i|) + \beta \sqrt{\ell \Dt}\right) \\
    &\leq \Dt \beta^{\prime} \hat{D} \left(\max\limits_{1 \leq i \leq v} (|\bX^0_i| + |\bY^0_i|) + \max\limits_{1 \leq i \leq v} |\overline{\bX}_i| + \beta \sqrt{N\Dt}\right) 
    \end{aligned}
\end{equation}
Multiplying \eqref{eq:gd3}, \eqref{eq:gd6} and \eqref{eq:gd8} and using the product rule \eqref{eq:prod} yields the desired upper bound \eqref{eq:gbd},

\end{proof}
\subsubsection{Proof of Proposition \ref{prop:glb}}
To investigate how small the gradients in \eqref{eq:chain} can be, we will need the following \emph{order} notation:
\begin{equation}
    \label{eq:ord}
    \begin{aligned}
     {\bf \beta} &= \ord(\alpha), {\rm for}~\alpha,\beta \in \R_+ \quad {\rm if~there~exists~constants}~ \overline{C},\underline{C}~{\rm such~that}~\underline{C} \alpha \leq \beta \leq \overline{C} \alpha. \\
   {\bf M} &= \ord(\alpha), {\rm for}~{\bf M} \in \R^{d_1 \times d_2}, \alpha \in \R_+ \quad {\rm if~there~exists~constant}~ \overline{C}~{\rm such~that}~\|{\bf M}\| \leq \overline{C} \alpha.
   \end{aligned}
\end{equation}
 Equipped with this notation, we proceed to prove Proposition \ref{prop:glb} below,

\begin{proof}
The key ingredient in the proof is the following representation formula,
\begin{equation}
    \label{eq:rform}
    \frac{\partial \bZ^N}{\partial \bZ^\ell} =  \ident_{2v \times 2v} + \Dt \sum\limits_{n=\ell+1}^N \bE^{n,n-1} + \ord(\Dt^2),
    \end{equation}
the proof of which follows directly from the identity \eqref{eq:grad6} and the boundedness of the matrices $\bE,\bF$ in \eqref{eq:grad6}. 

Then, \eqref{eq:gform} follows from a multiplication of \eqref{eq:gd4}, \eqref{eq:rform} and \eqref{eq:gd7} and a straightforward rearrangement of the terms, 
    
\end{proof}
One readily observes from the formula \eqref{eq:gform}, that to leading order in the small parameter $\Dt$, the gradient $\frac{\partial \cJ}{\partial \bw^\ell_k}$ is \emph{independent} of the number of layers $N$ of the underlying GNN. Thus, although the gradient can be small (due to small $\Dt$), it will not vanish by increasing the number of layers, mitigating the vanishing gradient problem. Even if small parameter $\Dt$ depends on the number of layers, as long as this dependence is polynomial i.e., $\Dt \sim N^{-s}$, for some $s$, the gradient cannot decay exponentially in $N$, alleviating the vanishing gradients problem in this case too. 

\end{document}